\documentclass[11pt]{article}
\usepackage[letterpaper]{geometry}
\usepackage[parfill]{parskip}
\usepackage{amsmath,amsthm,amssymb,bbm}
\usepackage{mathtools}
\usepackage{cases}

\usepackage[english]{babel}
\usepackage{dsfont}

\usepackage{microtype}
\usepackage{comment}
\usepackage{subfigure}
\usepackage{algorithm,algorithmic}
\usepackage{color}
\usepackage{appendix}

\usepackage{authblk}

\usepackage{url}
\usepackage[authoryear]{natbib}
\usepackage[colorlinks,citecolor=blue,urlcolor=blue,linkcolor=blue,linktocpage=true]{hyperref}
\usepackage{cleveref}
\crefformat{equation}{(#2#1#3)}
\crefrangeformat{equation}{(#3#1#4) to~(#5#2#6)}
\crefname{equation}{}{}
\Crefname{equation}{}{}

\crefname{definition}{\textbf{definition}}{definitions}
\Crefname{definition}{Definition}{Definitions}
\crefname{assumption}{\textbf{assumption}}{assumptions}
\Crefname{assumption}{Assumption}{Assumptions}

\definecolor{maroon}{RGB}{192,80,77}
\newcommand{\maroon}[1]{\textcolor{maroon}{#1}}
\newcommand{\explain}[2]{\underset{\mathclap{\overset{\uparrow}{#2}}}{#1}}
\newcommand{\explainup}[2]{\overset{\mathclap{\underset{\downarrow}{#2}}}{#1}}

\newtheorem{theorem}{Theorem}
\newtheorem{lemma}[theorem]{Lemma}
\newtheorem{proposition}[theorem]{Proposition}

\newtheorem{corollary}[theorem]{Corollary}
\newtheorem{definition}[theorem]{Definition}
\newtheorem{example}[theorem]{Example}

\newcommand{\argmin}{\mathop{\mathrm{argmin}}}

\newcommand{\minimize}{\mathop{\mathrm{minimize}}}

\def\E{\mathbb{E}}
\def\P{\mathbb{P}}

\def\R{\mathbb{R}}
\def\cA{\mathcal{A}}
\def\cB{\mathcal{B}}
\def\cD{\mathcal{D}}

\def\cF{\mathcal{F}}

\def\cH{\mathcal{H}}
\def\cI{\mathcal{I}}

\def\cP{\mathcal{P}}
\def\cR{\mathcal{R}}
\def\cS{\mathcal{S}}

\def\cX{\mathcal{X}}
\def\cY{\mathcal{Y}}
\def\cZ{\mathcal{Z}}

\begin{document}

\title{Learning with Differential Privacy: Stability, Learnability and the Sufficiency and Necessity of ERM Principle}

\author[1,2]{Yu-Xiang Wang}
\author[2]{Jing Lei}
\author[1,2]{Stephen E. Fienberg}
\affil[1]{Machine Learning Department, Carnegie Mellon University}
\affil[2]{Department of Statistics, Carnegie Mellon University}

\maketitle

\begin{abstract}
While machine learning has proven to be a powerful data-driven solution to many real-life problems, its use in sensitive domains has been limited due to privacy concerns. 
A popular approach known as \emph{differential privacy} offers provable privacy guarantees, but it is often observed in practice that it could substantially hamper learning accuracy.
In this paper we study the learnability (whether a problem can be learned by any algorithm) under Vapnik's general learning setting with differential privacy constraint, and reveal some intricate relationships between privacy, stability and learnability.

In particular, we show that a problem is privately learnable \emph{if an only if} there is a private algorithm that asymptotically minimizes the empirical risk (AERM). In contrast, for non-private learning AERM alone is not sufficient for learnability. This result suggests that when searching for private learning algorithms, we can restrict the search to algorithms that are AERM. In light of this, we propose a conceptual procedure that always finds a universally consistent algorithm whenever the problem is learnable under privacy constraint. We also propose a generic and practical algorithm and show that under very general conditions it privately learns a wide class of learning problems. 
Lastly, we extend some of the results to the more practical $(\epsilon,\delta)$-differential privacy and establish the existence of a phase-transition on the class of problems that are approximately privately learnable with respect to how small $\delta$ needs to be.

\end{abstract}

\newpage
\tableofcontents
\newpage

\section{Introduction}
Increasing public concerns regarding data privacy have posed obstacles in the development and application of new
machine learning methods as data collectors and curators may no longer be able to share
data for research purposes.     In addition to addressing the original goal of information extraction, privacy-preserving learning also requires
the learning procedure to protect sensitive information of individual data entries.
For example, the second Netflix Prize competition was canceled in response to a lawsuit and Federal Trade Commission privacy concerns, and the National Institute of Health decided in August 2008 to remove aggregate Genome-Wide Association Studies (GWAS) data from the public web site, after learning about a potential privacy risk. 

A major challenge in developing privacy-preserving learning methods is to  quantify formally
the amount of privacy leakage, given all possible and unknown auxiliary information the attacker may have, a challenge in part addressed by the notion of \emph{differential privacy} \citep{dwork2006differential,DworkMNS06}.
Differential privacy  has three main advantages over other approaches: (1)  it rigorously
quantifies the privacy property of any data analysis mechanism; (2)  it
 controls the amount of privacy leakage regardless of the attacker's resource or knowledge, (3)  it has useful interpretations from the perspectives of Bayesian inference and statistical hypothesis testing, and hence fits naturally in the general framework of statistical machine learning, e.g., see \citep{dwork2009,wasserman2010,smith2011,lei2011,wang2015}, as well as applications involving regression
 \citep{chaudhuri2011differentially,thakurta2013} and GWAS data \citep{yu2014}, etc.

In this paper we focus on the following fundamental question about differential privacy and
machine learning:
{\it  What problems can we learn with differential privacy?}
Most literature focuses on designing differentially private extensions of various learning algorithms, where the methods depend crucially on the specific context and differ vastly in nature. But with the privacy constraint, we have less choice in developing learning and data analysis algorithms.  It remains unclear how such a constraint affects our ability to learn, and if it is possible to design a generic privacy-preserving analysis mechanism that is applicable to a wide class of learning problems.
\paragraph{Our Contributions} We provide a general answer to the relationship between learnability and differential privacy under Vapnik's General Learning Setting \citep{vapnik1995nature} in four aspects.

1. We characterize the subset of problems in the General Learning Setting that can be learned under differential privacy. Specifically, we show that a sufficient and necessary condition for a problem to be privately learnable is the existence of an algorithm that is differentially private and asymptotically minimizes the empirical risk. This characterization generalizes previous studies of the subject \citep{kasi2011,beimel2013characterizing} that focus on binary classification in discrete domain under the PAC learning model. Technically, the result relies on the now well-known intuitive observation that ``privacy implies algorithmic stability'' and the argument in \citet{shalev2010learnability} that shows a variant of algorithmic stability is necessary for learnability.

2. We also introduce a weaker notion of learnability, which only requires consistency for a class of distributions $\mathfrak{D}$. Problems that are not privately learnable (a surprisingly large class that includes simple problems such as 0-1 loss binary classification in continuous feature domain \citep{chaudhuri2011sample}) are usually private $\mathfrak{D}$-learnable for some ``nice'' distribution class $\mathfrak{D}$.  We characterize the subset of private $\mathfrak{D}$-learnable problems that are also (non-privately) learnable using conditions analogous to those in distribution-free private learning.

3. Inspired by the equivalence between privacy learnability and private AERM, we propose a generic (but impractical) procedure that always finds a consistent and private algorithm for any privately learnable (or $\mathfrak{D}$-learnable) problems. We also study a specific algorithm that aims at minimizing the empirical risk while preserving the privacy. We show that under a sufficient condition that relies on the geometry of the hypothesis space and the data distribution, this algorithm is able to privately learn (or $\mathfrak{D}$-learn) a large range of learning problems including classification, regression, clustering, density estimation and etc, and it is computationally efficient when the problem is convex. In fact, this generic learning algorithm learns any privately learnable problems in the PAC learning setting \citep{beimel2013characterizing}. It remains an open problem whether the second algorithm also learns any privately learnable problem in the General Learning Setting.

4. Lastly, we provide a preliminary study of learnability under the more practical $(\epsilon,\delta)$-differential privacy. Our results reveal that whether there is separation between learnability and approximate private learnability depends on how fast $\delta$ is required to go to $0$ with respect to the size of the data. Finding where the exact phase transition occurs is an open problem of future interest.

Our primary objective is to understand the conceptual impact of differential privacy and learnability under a general framework and the rates of convergence obtained in the analysis may be suboptimal. Although we do provide some discussion on polynomial time approximations to the proposed algorithm, learnability under computational constraints is beyond the scope of this paper. 


\paragraph{Related work} 
While a large amount of work has been devoted to finding consistent (and rate optimal) differentially private learning algorithms in various settings \citep[e.g.,][]{chaudhuri2011differentially,kifer2012private,jain2013differentially,bassily2014private},
the characterization of privately learnable problems were only studied in a few special cases.

\citet{kasi2011} showed that, for binary classification with a finite discrete hypothesis space, anything that is non-privately learnable is privately learnable under the agnostic Probably Approximately Correct (PAC) learning framework, therefore ``finite VC-dimension'' characterizes the set of private learnable problems in this setting. \citet{beimel2013characterizing} extends  \citet{kasi2011} by characterizing the sample complexity of the same class of problems, but the result only applies to the realizable (non-agnostic) case.  \citet{chaudhuri2011sample} provided a counter-example showing that for continuous hypothesis space and data space, there is a gap between learnability and learnability under privacy constraint. They proposed to fix this issue by either weakening the privacy requirement to labels only or by restricting the class of potential distribution. While meaningful in some cases, these approaches do not resolve the learnability problem in general.

A key difference of our work from \citet{kasi2011,chaudhuri2011sample,beimel2013characterizing} is that we consider a more general class of learning problems and provide a proper treatment in a statistical learning framework. This allows us to capture a wider collection of
important learning problems (see Figure \ref{fig:gls_examples} and Table~
\ref{tab:examples_gls}).

It is important to note that despite its generality, Vapnik's general learning setting still does not nearly cover the full spectrum of private learning.
In particular, our results do not apply to improper learning (learning using a different hypothesis class) as considered in \citet{beimel2013characterizing} or to structural loss minimization (the loss function jointly take all data points as input) considered in \citet{beimel2013private}. Also, our results do not address the sample complexity problem, which remains open in the general learning setting even for learning without privacy constraints.

Our characterization of private learnability (and private $\mathfrak{D}$-learnability) in Section~\ref{sec:characterization} uses a recent advance
in the characterization of  general learnability given by \citet{shalev2010learnability}.
Roughly speaking, they showed that a problem is learnable if and only if there exists an algorithm that (i) is stable under small perturbation of training data, and (ii) behaves like empirical risk minimization (ERM) asymptotically. We also makes use of a folklore observation that ``Privacy $\Rightarrow$ Stability $\Rightarrow$ Generalization''. The connection of privacy and stability appeared as early as 2008 in a conference version of \citet{kasi2011}. Further connection to ``generalization'' recently appeared in blog posts\footnote{For instance, Frank McSherry described in a blog post an example of exploiting differential privacy for measure concentration \url{http://windowsontheory.org/2014/02/04/differential-privacy-for-measure-concentration/}; Moritz Hardt discussed the connection of differential privacy to stability and generalization in his blog post \url{http://blog.mrtz.org/2014/01/13/false-discovery}.}, stated as a theorem in Appendix F of \citet{bassily2014private}, and was shown to hold with strong concentration in \citet{dwork2014preserving}.

	\citet{dwork2014preserving} is part of an independent line of work \citep{hardt2014preventing,bassily2015algorithmic,dwork2015reusable,blum2015ladder} on adaptive data analysis, which also stems from the observation that privacy implies stability and generalization. Comparing to adaptive data analysis works, our focus is quite different. 
  Adaptive data analysis work focus on the impact of $k$ on how fast the maximum absolute error of $k$-adaptively chosen queries goes to $0$ as a function of $n$, while this paper is concerned with whether the error can go to $0$ at all for each learning problem
  when we require the learning algorithm be differentially private with $\epsilon < \infty$.
Nonetheless, we acknowledge that Theorem~7 in \citet{dwork2014preserving} provides an interesting alternative proof for ``differentially private learners have small generalization error'', when choosing the statistical query as evaluating a loss function at a privately learned hypothesis. The connection is not quite obvious and we provide a more detailed explanation in Appendix~\ref{app:B}.

The main tool used in the construction of our generic private learning algorithm in Section~\ref{sec:pen-erm} is the Exponential Mechanism \citep{mcsherry2007mechanism}, which provides a simple and differentially-private approximation to the maximizer of a score function among a candidate set.  In the general learning context, we use the negative empirical risk as the utility function, and apply the exponential mechanism to a possibly pre-discretized
hypothesis space.
This exponential mechanism approach was used in \citet{bassily2014private} for minimizing convex and Lipschitz functions. The sample discretization procedure has been considered in \citet{chaudhuri2011sample} and \citet{beimel2013characterizing}. Our scope and proof techniques are different. Our strategy is to show that, under some general regularity conditions, the exponential mechanism is stable and behaves like ERM. Our sublevel set condition has the same flavor as that in the proof of \citet[Theorem~3.2]{bassily2014private}, although we do not need the loss function to be convex or Lipschitz.

Stability, privacy and generalization were also studied in \citet{thakurta2013} with different notions of stability. More importantly, their stability is used as an assumption rather than a consequence, so their result is not directly comparable to ours.

\section{Background}
\subsection{Learnability under the General Learning Setting}
\begin{figure}	
	\centering
\subfigure[Illustration of general learning setting. Examples of known DP extensions are circled in \textbf{\maroon{maroon}}.]{\label{fig:gls_examples}		
\includegraphics[width=0.9\textwidth]{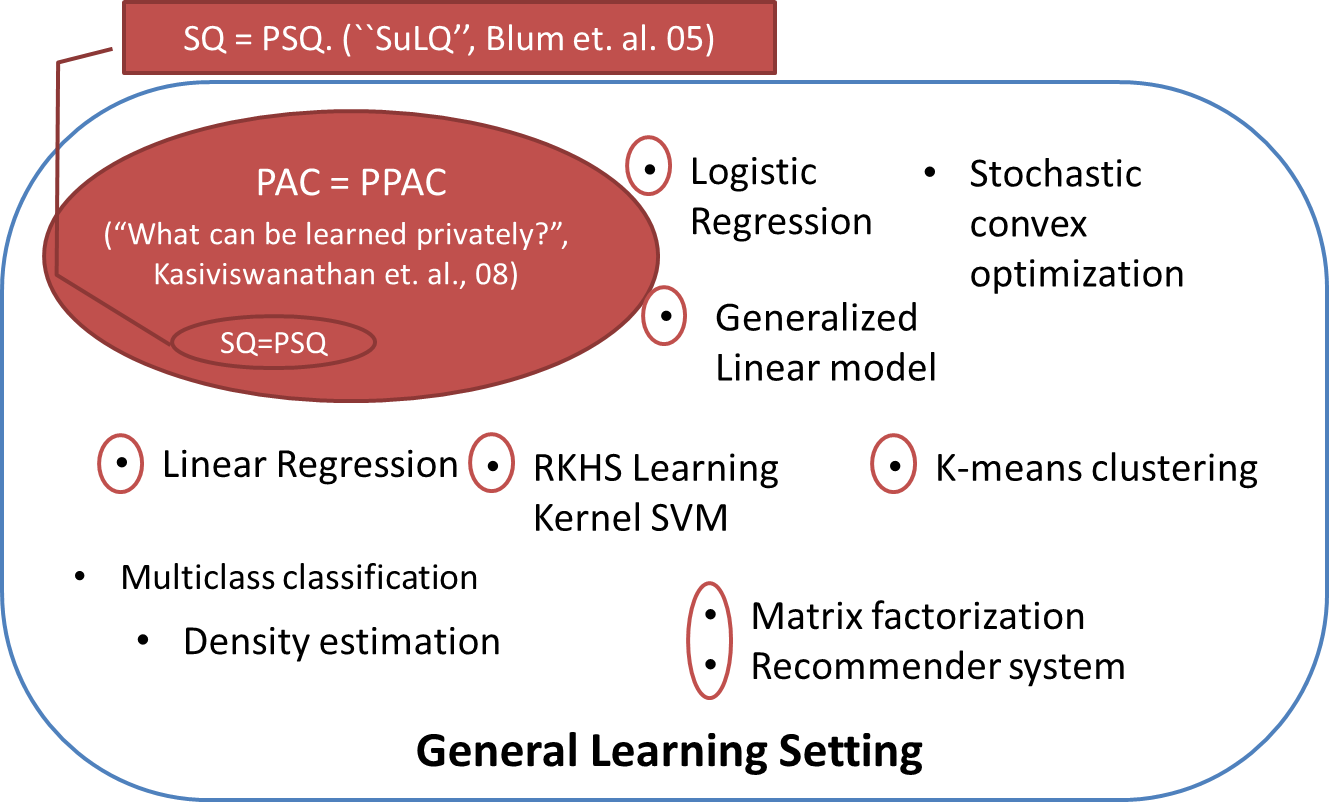}
}\\
\subfigure[Our characterization of private learnable problems in the general learning setting (in \textbf{{\color{blue} blue}}).]{\label{fig:gls_characterization}
\includegraphics[width=0.9\textwidth]{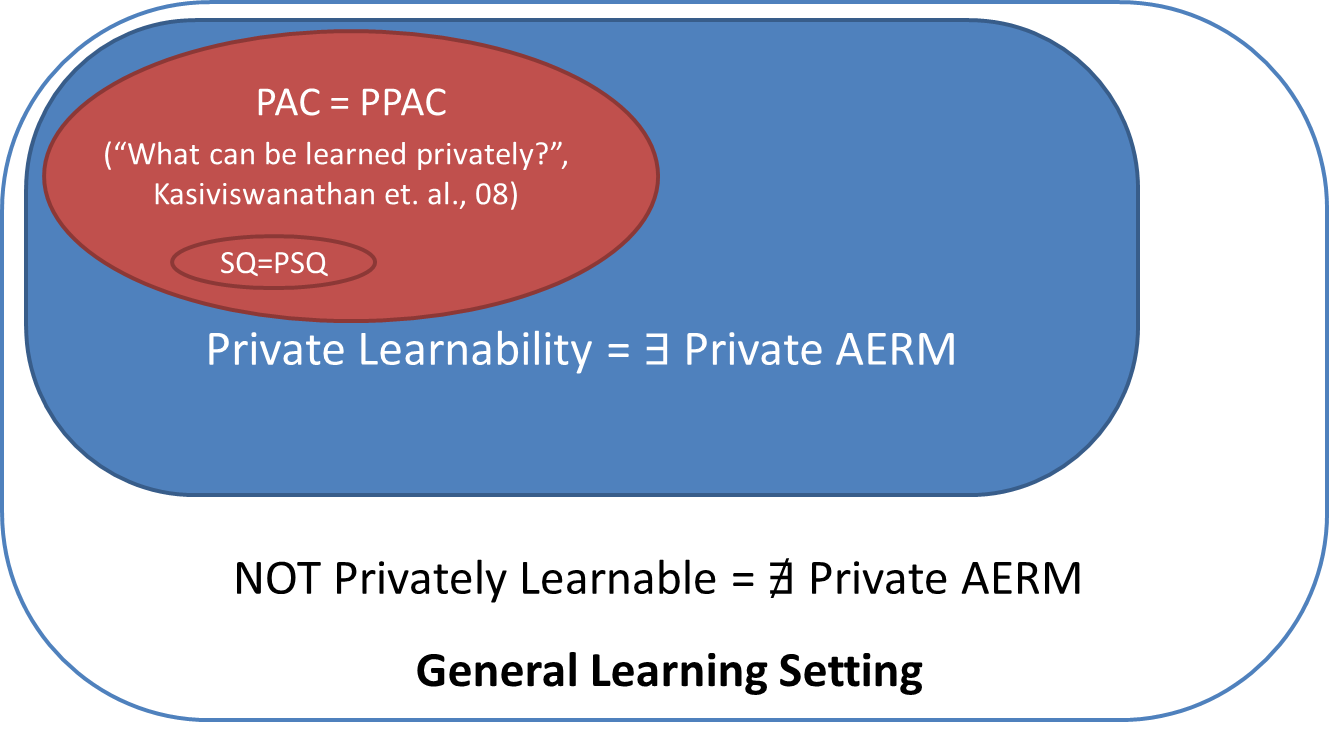}}
		\caption{The Big Picture: illustration of general learning setting and our contribution in understanding differentially private learnability. }\label{fig:gls_privacy}
\end{figure}

In the General Learning Setting of \citet{vapnik1995nature}, a learning problem is characterized
by a triplet $(\mathcal Z, \mathcal H, \ell)$.
Here $\mathcal{Z}$ is the sample space (with a $\sigma$-algebra).
The hypothesis space $\mathcal{H}$ is a collection of models such that each $h\in\mathcal{H}$ describes some structures of the data.
The loss function $\ell:\mathcal{H}\times \mathcal{Z}\rightarrow \R$ measures how well the hypothesis $h$ explains the data instance $z\in \cZ$. For example, in supervised learning problems $\cZ = \mathcal{X}\times \mathcal{Y}$ where $\cX$ is the feature space and $\cY$ is the label space; $\cH$ defines a collection of mapping $h: \mathcal{X}\rightarrow \mathcal{Y}$; and $\ell(h,z)$ measures how well $h$ predicts the feature-label relationship $z=(x,y)\in \cZ$.  This setting includes problems with continuous input/output in potentially infinite dimensional spaces (e.g. RKHS methods), hence is much more general than PAC learning.
In addition, the general learning setting also covers a variety of unsupervised learning problems, including clustering, density estimation, principal component analysis (PCA) and variants (e.g., Sparse PCA, Robust PCA), dictionary learning, matrix factorization and even Latent Dirichlet Allocation (LDA). Details of these examples are given in Table~\ref{tab:examples_gls} (the first few are extracted from \citet{shalev2010learnability}).

To account for the randomness in the data, we are primarily interested in the case where the data  $Z=\{z_1,...,z_n\}\in \mathcal Z^n$ are independent samples drawn from an unknown probability
distribution $\cD$ on $\cZ$. We denote such a random sample by $Z\sim \cD^n$.
For a given distribution $\cD$, let $R(h)$ be the expected loss of hypothesis $h$  and $\hat{R}(h,Z)$ the empirical risk from a sample $Z\in\cZ^{n}$:
\begin{align*}
R(h)=\E_{z\sim\mathcal{D}} \ell(h,z), &&\hat{R}(h,Z)=\frac{1}{n}\sum_{i=1}^n \ell(h,z_i)\,.
\end{align*}
The optimal risk $R^*=\inf_{h\in\mathcal{H}} R(h)$ and we assume that it is achieved by an optimal $h^*\in \mathcal{H}$. Similarly, the minimal empirical risk $\hat R^*(Z)=\inf_{h\in\mathcal{H}} \hat R(h,Z)$ is achieved by $\hat h^*(Z)\in\mathcal H$.
For a possibly randomized algorithm $\cA:\mathcal{Z}^n\rightarrow \mathcal{H}$  that learns some hypothesis $\cA(Z)\in \mathcal{H}$ given data sample $Z$,  we say $\cA$ is \emph{consistent} if
\begin{equation}\label{eq:def_cons}
\lim_{n\rightarrow\infty}\E_{Z\sim \cD^n}  \left( \E_{h\sim \cA(Z)} R(h) - R^*   \right) = 0.
\end{equation}
In addition, we say $\cA$ is consistent with rate $\xi(n)$ if
\begin{equation}\label{eq:def_cons_wrate}
\E_{Z\sim \cD^n}  \left( \E_{h\sim \cA(Z)} R(h) - R^*   \right)  \leq  \xi(n), \,\text{ where } \lim_{n\rightarrow\infty}\xi(n)\rightarrow 0.
\end{equation}

Since the distribution $\cD$ is unknown, we cannot adapt the algorithm $\cA$ to $\cD$, especially when privacy is a concern. Also, even if $\cA$ is pointwise consistent for any distribution $\cD$, it may have different rates for different $\cD$ and potentially be arbitrarily slow for some $\cD$. This makes it hard to evaluate whether $\cA$ indeed learns the learning problem and forbids the study of the learnability problem. In this study, we adopt the stronger notion of learnability considered in \citet{shalev2010learnability}, which is a direct generalization of PAC-learnability \citep{valiant1984theory} and agnostic PAC-learnability \citep{kearns1992toward} to the General Learning Setting as studied by \citet{haussler1992decision}.
\begin{definition}[Learnability, \citealp{shalev2010learnability}]\label{def:learnability}
A learning problem is learnable if there exists an algorithm $\mathbb{\cA}$ and rate $\xi(n)$, such that $\cA$ is consistent with rate $\xi(n)$ for any distribution $\mathcal{D}$ defined on $\mathcal{Z}$.
\end{definition}
This definition requires consistency to hold universally for any distribution $\cD$ with a uniform (distribution-independent) rate $\xi(n)$. This type of problem is often called \emph{distribution-free learning} \citep{valiant1984theory}, and an algorithm is said to be \emph{universally consistent} with rate $\xi(n)$ if it realizes the criterion.



\begin{table}
{\footnotesize
\centering
\begin{tabular}{|l|l|l|l|}
  \hline
   Problem & Hypothesis class $\cH$ & $\cZ$ or $\cX\times\cY$ & Loss function $\ell$  \\\hline
 Binary classification & $\cH\subset \{f: \{0,1\}^d\rightarrow\{0,1\}\} $& $\{0,1\}^d\times\{0,1\}$ & $1(h(x)\neq y)$\\
  Regression & $\cH\subset \{f: [0,1]^d\rightarrow \R\}$ & $[0,1]^d\times\R$ & $|h(x) - y|^2$\\
  Density Estimation & Bounded distributions on $\cZ$ & $\cZ\subset \R^d$ & $-\log(h(z))$ \\
  K-means Clustering &
$\{S\subset \R^d : |S|=k\}$& $\cZ\subset \R^d$ & $\underset{c
  \in h}{\min}\|c-z\|^2$ \\
  RKHS classification & Bounded RKHS & RKHS$\times\{0,1\}$ & $\max\{0,1-y\langle x,h\rangle\}$ \\
  RKHS regression & Bounded RKHS& RKHS$\times\R$ & $|\langle x,h\rangle - y|^2$\\
  Sparse PCA & Rank-$r$ projection matrices & $\R^d$  & $\|h z - z\|^2 + \lambda\|h\|_{1}$  \\
  Robust PCA & All subspaces in $\R^d$ & $\R^d$  & $\|\cP_h(z)-z\|_{1} + \lambda \mathrm{rank}(h)$  \\
  Matrix Completion & All subspaces in $\R^d$ & $\R^d\times \{1,0\}^d$  & $\underset{b\in h}{\min}\|y\circ(b-x)\|^2 + \lambda\mathrm{rank}(h)$  \\
 Dictionary Learning & All dictionaries $\in \R^{d\times r}$ & $\R^d$ & $\underset{b\in \R^r}{\min}\|hb - z\|^2 + \lambda\|b\|_1$\\
  Non-negative MF & All dictionaries $\in \R_+^{d\times r}$ & $\R^d$ & $\underset{b\in \R_+^r}{\min}\|hb - z\|^2$\\
  Subspace Clustering & A set of $k$ rank-$r$ subspaces  & $\R^d$  & $\underset{b\in h}{\min}\|\cP_b(z) - z\|^2$\\
 Topic models (LDA) & $\{\P(\text{word} | \text{topic})\}$\ & Documents & $\underset{\tiny b\in\{\P(\text{Topic})\}}{-\;\;\mathrm{max}\quad}\
 \underset{w\in z}{\sum}\log\P_{b,h}(w)$\\
  \hline
\end{tabular}
}
\caption{An illustration of problems in the General Learning setting.}\label{tab:examples_gls}
\end{table}

\subsection{Differential privacy}
Differential privacy requires that if we arbitrarily perturb a database by only one data point, the output should not differ much. Therefore, if one conducts a statistical test for whether any individual is in the database or not, the false positive and false negative probabilities cannot both be small \citep{wasserman2010}. Formally, define ``Hamming distance''
\begin{equation}\label{eq:Z_distance}
  d(Z,Z^\prime):= \# \{i=1,...,n: z_i\neq z_i^{\prime}\}\,.
\end{equation}

\begin{definition}[$\epsilon$-Differential Privacy, \citealp{dwork2006differential}]\label{def:diff_privacy}
An algorithm $\cA$ is $\epsilon$-differentially private, if
$$
\P(\cA(Z)\in H)\leq \exp(\epsilon)\P(\cA(Z^{\prime})\in H)
$$
for $\forall\ Z,\ Z^{\prime}$ obeying $d(Z,Z^{\prime})=1$ and any measurable subset $H\subseteq \mathcal H$.
\end{definition}
There are weaker notions of differential privacy. For example $(\epsilon,\delta)$-differential privacy allows for a small probability $\delta$ where the privacy guarantee does not hold. In this paper, we will mainly work with the stronger $\epsilon$-differential privacy. In \Cref{sec:eps_delta} we discuss the problem of $(\epsilon,\delta)$-differential privacy and extend some of the results to this setting.

Our objective is to understand whether there is a gap between learnable problems and privately learnable problems in the general learning setting, and to quantify the tradeoff required to protect privacy. To achieve this objective, we need to show the existence of an algorithm that learns a class of problems while preserving differential privacy. More formally, we define
\begin{definition}[Private learnability]\label{def:private-learnability}
	A learning problem is privately learnable with rate $\xi(n)$ if there exists an algorithm $\cA$ that satisfies both universal consistency (as in Definition~\ref{def:learnability}) with rate $\xi(n)$ and $\epsilon$-differential privacy with privacy parameter $\epsilon < \infty$.
\end{definition}

We can view the consistency requirement Definition~\ref{def:private-learnability}  as a measure of utility. This utility is not a function of the observed data, however, but rather how the results generalize to unseen data.


The following lemma shows that the above definition of private learnability is actually equivalent to a seemingly much stronger condition with a vanishing privacy loss $\epsilon$.
	
	\begin{lemma}\label{lem:dp_learnability}
		If there is an $\epsilon$-DP algorithm that is consistent with rate $\xi(n)$ for some constant $0<\epsilon<\infty$, then there is a $\frac{2}{\sqrt{n}}\left(e^\epsilon-e^{-\epsilon}\right)$-DP algorithm that is consistent with rate $\xi(\sqrt{n})$.
	\end{lemma}
	The proof, given in \Cref{sec:proof-subsampling}, uses a subsampling theorem adapted from \citet[Lemma~4.4]{beimel2014bounds}. 


There are many approaches to design differentially private algorithms, such as noise perturbation
using Laplace noise \citep{dwork2006differential,DworkMNS06} and the Exponential Mechanism
\citep{mcsherry2007mechanism}.  Our construction of generic differentially private learning
algorithms applies the Exponential Mechanism to penalized empirical risk minimization.  Our argument will make use of a general characterization of learnability described below.


%
%

\subsection{Stability and Asymptotic ERM}
An important breakthrough in learning theory is a full characterization of all learnable problems in the General Learning Setting in terms of stability and empirical risk minimization \citep{shalev2010learnability}.
Without assuming uniform convergence of empirical risk, \citeauthor{shalev2010learnability} showed that a problem is learnable if and only if there exists a ``strongly uniform-RO stable'' and ``always asymptotically empirical risk minimization'' (Always AERM) randomized algorithm that learns the problem. Here ``RO'' stands for ``replace one''. Also, any strongly uniform-RO stable and ``universally'' AERM (weaker than ``always'' AERM) learning rule learns the problem consistently.  Here we give detailed definitions.

\begin{definition}[Universally/Always AERM, \citealp{shalev2010learnability}]\label{def:always-aerm}
A (possibly randomized) learning rule $\mathcal A$ is Universally AERM if for any distribution $\cD$ defined on domain $\cZ$
  $$
  \E_{Z\sim \cD^n}\left[\E_{h\sim \mathcal A(Z)} \hat R(h,Z) - \hat R^*(Z)\right]\rightarrow 0,~~
  \text{as}~~n\rightarrow\infty\,
  $$
    where $\hat R^*(Z)$ is the minimum empirical risk for data set $Z$.   We say $\mathcal A$ is Always AERM, if in addition,
    $$
  \sup_{Z\in \mathcal Z^n}\E_{h\sim \mathcal A(Z)} \hat R(h,Z) - \hat R^*(Z)\rightarrow 0,~~
  \text{as}~~n\rightarrow\infty\,.
  $$
\end{definition}

\begin{definition}[Strongly Uniform RO-Stability, \citealp{shalev2010learnability}]\label{def:stability}
An \\algorithm $\cA$ is strongly uniform RO-stable if
$$
\sup_{z \in \cZ} \sup_{\tiny{\begin{array}{c}
                         Z,Z^\prime\in \cZ^n, \\
                         d(Z,Z^\prime)=1
                       \end{array}}
  } |\E_{h\sim \cA(Z)}\ell(h,z) - \E_{h\sim\cA(Z^{\prime})} \ell(h,z)| \rightarrow 0 \text{ as }n\rightarrow \infty.
$$
where $d(Z,Z^\prime)$ is defined in \eqref{eq:Z_distance}, in other word, $Z$ and $Z^{\prime}$ can differ by at most one data point.
\end{definition}
Since we will not deal with other variants of algorithmic stability in this paper (e.g., hypothesis stability~\citep{kearns1999algorithmic}, uniform stability~\citep{bousquet2002stability} and leave-one-out (LOO) stability in~\citet{mukherjee2006learning}),  we simply call Definition~\ref{def:stability} stability or uniform stability. Likewise, we will refer to $\epsilon$-differential privacy as just ``privacy'' although there are several other notions of privacy in the literature.

\section{Characterization of private learnability}\label{sec:characterization}
We are now ready to state our main result. The only assumption we make is the uniform boundedness of the loss function. This is also assumed in \citet{shalev2010learnability} for the learnability problem without privacy constraints. Without loss of generality, we can assume $0 \leq \ell(h,z) \leq 1$.
\begin{figure}
  \centering
  \includegraphics[width=0.7\textwidth]{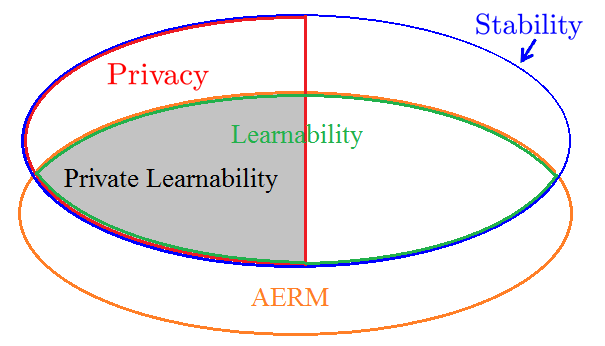}\\
  \caption{A summary of the relationships of various notions revealed by our analysis.}\label{fig:blockdiagraom}
\end{figure}



\begin{theorem}\label{thm:characterization}
Given a learning problem $(\cZ,\cH,\ell)$,  the following statements are equivalent.
\begin{enumerate}
  \item The problem is privately learnable.
  \item There exists a differentially private universally AERM algorithm.
  \item There exists a differentially private always AERM algorithm.
\end{enumerate}
\end{theorem}

The proof is simple yet revealing, we will present the arguments for $2 \Rightarrow 1$
(sufficiency of AERM) in Section~\ref{sec:privacy_stability} and $1 \Rightarrow 3$ (necessity of AERM) in Section~\ref{sec:necessity}. $3 \Rightarrow 2$ follows trivially from the definition of ``always'' and ``universal'' AERM.

The theorem says that we can stick to ERM-like algorithms
for private learning, despite that ERM may fail for some problems in the (non-private) general learning setting \citep{shalev2010learnability}. Thus a standard procedure for finding universally consistent and differentially private algorithms would be to approximately minimize the empirical risk using some differentially private procedures~\citep{chaudhuri2011differentially,kifer2012private,bassily2014private}. 
If the utility analysis reveals that the method is AERM, we do not need to worry about generalization as it is guaranteed by privacy. This consistency analysis is considerably simpler than non-private learning problems where one typically needs to control generalization error either via uniform convergence (VC-dimension, Rademacher complexity, metric entropy, etc) or to adopt the stability argument \citep{shalev2010learnability}.

This result does not imply that privacy is helping the algorithm to learn in any sense, as the simplicity is achieved at the cost of having a smaller class of learnable problems. A concrete example of a problem being learnable but not privately learnable is given in \citep{chaudhuri2011sample} and we will revisit it in Section~\ref{sec:counter_example}.
For some problems where ERM fails, it may not be possible to make it AERM while preserving privacy. In particular, we were not able to privatize the problem in Section~4.1 of \citet{shalev2010learnability}.

To avoid any potential misunderstanding, we stress that Theorem~\ref{thm:characterization} is a characterization of learnability, \emph{not} learning algorithms. It does not prevent the existence of a universally consistent learning algorithm that is private but not AERM.
Also, the characterization given in Theorem~\ref{thm:characterization}  is about consistency, and it does not claim anything on sample complexity. An algorithm that is AERM may be suboptimal in terms of convergence rate.




\subsection{Sufficiency: Privacy implies stability}\label{sec:privacy_stability}
A key ingredient in the proof of sufficiency is a well-known heuristic observation that differential privacy by definition implies uniform stability, which is useful in its own right.
\begin{lemma}[Privacy $\Rightarrow$ Stability]\label{lem:EM_stability}
Assume $0\leq\ell(h,z)\leq 1$, any $\epsilon$-differentially private algorithm satisfies $(e^\epsilon-1)$-stability. Moreover if $\epsilon\leq 1$ it satisfies $2\epsilon$-stability.
\end{lemma}
The proof of this lemma comes directly from the definition of differential privacy so it is algorithm independent. The converse, however, is not true in general (e.g., a non-trivial deterministic algorithm can be stable, but not differentially private.)


\begin{corollary}[Privacy + Universal AERM $\Rightarrow$ Consistency]\label{corr:consistency}
If a learning algorithm $\cA$ is $\epsilon(n)$-differentially private and   $\cA$ is universally AERM with rate $\xi(n)$, then $\cA$ is universally consistent with rate $\xi(n) + e^{\epsilon(n)}-1=O(\xi(n) + \epsilon(n))$.
\end{corollary}%


The proof of Corollary \ref{corr:consistency}, provided in the Appendix, combines Lemma~\ref{lem:EM_stability} and the fact that consistency is implied by stability and AERM (Theorem~\ref{thm:stable_aerm}).  Our Theorem~\ref{thm:stable_aerm} is based on minor modifications of Theorem 8 in  \citet{shalev2010learnability}. In fact, Corollary~\ref{corr:consistency} can be stated in a stronger per distribution form, since universality is not used in the proof. We will revisit this point when we discuss a weaker notion of private learnability below.

Lemma~\ref{lem:dp_learnability} and Corollary~\ref{corr:consistency} together establishes $2 \Rightarrow 1$ in Theorem~\ref{thm:characterization}.

If for a problem privacy and always AERM cannot coexist, then the problem is not privately learnable. This is what we will show next.


\subsection{Necessity: Consistency implies Always AERM}\label{sec:necessity}
To prove that the existence of an always AERM learning algorithm is necessary for any private learnable problems, it suffices to construct such a learning algorithm from 

or each learnable problem.
any universally consistent learning algorithm.

\begin{lemma}[Consistency + Privacy $\Rightarrow$ Private Always AERM]\label{lem:necessity}
If $\cA$ is a universally consistent learning algorithm satisfying $\epsilon$-DP with any $\epsilon>0$ and consistent with rate $\xi(n)$, then there is another universally consistent learning algorithm $\cA'$ that is always AERM with rate $\xi(\sqrt{n})$ and satisfies $\frac{2}{\sqrt{n}}(e^{\epsilon}-e^{-\epsilon})$-DP.
\end{lemma}
Lemma~\ref{lem:necessity} is proved in \Cref{sec:proof-characterization}. The proof idea is to run $\cA$ on a size $O(\sqrt{n})$ random subsample of $Z$, which will be universally consistent with a slower rate, differentially private with $\epsilon(n) \rightarrow 0$ (Lemma~\ref{lem:sampling_thm}), and at the same time always AERM. The last part uses an argument in Lemma~24 of \citet{shalev2010learnability} which appeals to the universality of $\cA$'s consistency on a specific discrete distribution supported on the given data set $Z$.

	As pointed out by an anonymous reviewer, there is a simpler proof by invoking Theorem~10 of \citet{shalev2010learnability} that says any consistent and generalizing algorithm must be AERM and a result  \citep[e.g.,][Appendix F]{bassily2014private} that says ``privacy $\Rightarrow$ generalization''. This is a valid observation. But their Theorem~10 is proven using a detour through ``generalization'', which leads to a slower rate than what we are able to obtain in Lemma~\ref{lem:necessity} using a more direct argument.




\subsection{Private Learnability vs. Non-private Learnability}\label{sec:counter_example}
Now we have a characterization of all privately learnable problems,
a natural question to ask is that whether any learnable problem is also privately learnable. The answer is ``yes'' for learning in Statistical Query (SQ)-model and PAC Learning model (binary classification) with finite hypothesis space, and is ``no'' for continuous hypothesis space \citep{chaudhuri2011sample}.


By definition, all privately learnable problems are learnable. But now that we know that privacy implies generalization, it is tempting to hope that privacy can help at least some problem to learn better than any non-private algorithm.
In terms of learnability, the question becomes: Could there be a (learnable) problem that is  \emph{exclusively} learnable through private algorithms? We now show that such a problem does not exist.

\begin{proposition}\label{pro:non-private-learnability}
If a learning problem is learnable by an $\epsilon$-DP algorithm $\cA$, then it is also learnable by a non-private algorithm.
\end{proposition}
The proof is given in \Cref{sec:proof-counter-example}.  The idea is that $\cA(Z)$ defines a distribution over $\cH$. Pick an $z\in \cZ$. If $z\notin Z$, algorithm $\cA'=\cA$.
Otherwise, $\cA'(Z)$ samples from a slightly different distribution than $\cA(Z)$ that does not affect the expectation much.

On the other hand, not all learnable problems are privately learnable. This can already be seen from \citet{chaudhuri2011sample}, where the gap between learning and private learning is established. We revisit \citeauthor*{chaudhuri2011sample}'s example in our notation under the general learning setting and produce an alternative proof by showing that differential privacy contradicts \emph{always AERM}, then invoking Theorem~\ref{thm:characterization} to show the problem is not privately learnable.

\begin{proposition}[{\citealp[Theorem~5]{chaudhuri2011sample}}]\label{prop:counterexample}
There exists a problem that is learnable by a non-private algorithm, but not privately learnable. In particular, any private algorithm cannot be \emph{always AERM} in this problem.
\end{proposition}



We describe the counterexample and re-establish the impossibility of private learning
for this problem using
the contrapositive of Theorem~\ref{thm:characterization}, which suggests that if privacy and always AERM algorithm cannot coexist for some problem, then the problem is not privately learnable. 

Consider the binary classification problem with $\cX = [0,1]$, $\cY=\{0,1\}$ and 0-1 loss function. Let $\cH$ be the collection of threshold functions that output $h(x)=1$ if $x>h$ and $h(x)=0$ otherwise. This class has VC-dimension 1, and hence the problem is learnable.

Next we will construct $K=\lceil\exp(\epsilon_n n)\rceil$ data sets such that if $K-1$ of them obey AERM, the remaining one cannot be.
Let $\eta = 1/\exp(\epsilon n)$, $K:= \lceil 1/\eta\rceil$.
Let $h_1,h_2,...,h_K$ be a disjoint thresholds such that they are at least $\eta$ apart and $[h_i-\eta/3,h_i+\eta/3]$ are disjoint intervals.

If we take $Z_i\subseteq [h_i-\eta/3,h_i+\eta/3]$ with half of the points in $[h_i-\eta/3,h_i)$ and the other half in $(h_i,h_i+\eta/3]$ and we label each data point in it with $\mathbf 1(z>h_i)$, then empirical risk $\hat{R}(h_i,Z_i) =0 \;\forall i=1,...,K$. So for any AERM learning rule, $\E_{h\sim\cA(Z_i)} \hat{R}(h,Z_i) \rightarrow 0$ for all $i$. For some sufficiently large $n$, $\E_{h\sim\cA(Z_i)} \hat{R}(h,Z_i) < 0.1.$

Now consider $Z_1$,
$$\P (\cA(Z_1)\notin [h_1-\eta/3,h_1+\eta/3]) \geq \sum_{i=2}^{K}\P (\cA(Z_1)\in [h_i-\eta/3,h_i+\eta/3]), $$
since these intervals are disjoint. Then by the definition of $\epsilon$-DP,
\begin{equation}\label{eq:counter_exp_keyeq1}
\P(\cA(Z_1) \in [h_i-\eta/3,h_i+\eta/3]) \geq \exp(-\epsilon n)\P(\cA(Z_i) \in [h_i-\eta/3,h_i+\eta/3]).
\end{equation}
It follows that $\P(\cA(Z_i) \in [h_i-\eta/3,h_i+\eta/3]) > 0.9$ otherwise $\E_{h\sim\cA(Z_i)} \hat{R}(h,Z_i) \geq 0.1$, therefore \begin{equation}\label{eq:counter_exp_keyeq2}
\P (\cA(Z_1)\notin [h_1-\eta/3,h_1+\eta/3]) \geq K \exp(-\epsilon n) 0.9  \geq 0.9,
\end{equation}
and $\E_{h\sim\cA(Z_i)} \hat{R}(h,Z_i) \geq 0.9 \times 1 = 0.9$, which violates the ``always AERM'' condition that requires $\E_{h\sim\cA(Z_1)} \hat{R}(h,Z_1) < 0.1.$ Therefore, the problem is not privately learnable.


As is pointed out by an anonymous reviewer, the same conclusion of this impossibility result of privately learning thresholds on $[0,1]$ can be drawn numerically through the characterization of the sample complexity \citep{beimel2013characterizing}, via the bound that depends logarithmically on the $\log(|\cH|)$ and on $[0,1]$ this number is infinite. The above analysis provides different insights about the problem. We will be using it again for understanding the separation of learnability and learnability under $(\epsilon,\delta)$-Differential Privacy later in Section~\ref{sec:eps_delta}.

\subsection{Private $\mathfrak{D}$-learnability}
The above example implies that even very simple learning problems may not be privately learnable. To fix this caveat, note that most data sets of practical interest have nice distributions. Therefore, it makes sense to consider a smaller class of distributions, e.g., smooth distributions that have bounded $k$th order derivative, or those having bounded total variation. These are common assumptions in non-parametric statistics, such as kernel density estimation, smoothing spline regression and mode clustering. Similarly, in high dimensional statistics, there are often assumptions on the structures of the underlying distribution, such as sparsity, smoothness, and low-rank conditions.

\begin{definition}[(Private) $\mathfrak{D}$-learnability]
We say a learning problem  $(\cZ, \cH, \ell)$ is $\mathfrak{D}$-learnable if there exists a learning algorithm $\cA$ that is consistent for every unknown distribution $\cD \in \mathfrak{D}$.
If in addition, the problem is $\mathfrak{D}$-learnable under $\epsilon$-differential privacy for
some $0\le \epsilon<\infty$, then we say the problem is privately $\mathfrak{D}$-learnable.
\end{definition}

Almost all of our arguments hold in a per distribution fashion, therefore they also hold for any such
subclass $\mathfrak{D}$.  The only exception is the necessity of ``always AERM''
(Lemma~\ref{lem:necessity}), where
we used the universal consistency on an arbitrary discrete uniform distribution in the proof.
The characterization still holds if the class $\mathfrak{D}$ contains all finite discrete uniform
distributions. For general distribution classes, we
characterize private $\mathfrak{D}$-learnability using a weaker ``universally AERM'' (instead of ``always AERM'') under the assumption
that the problem itself is learnable in a distribution-free setting without privacy constraints.

\begin{lemma}[private $\mathfrak{D}$-learnability $\Rightarrow$ private $\mathfrak{D}$-universal AERM]\label{lem:necessity_mod}
If an $\epsilon$-DP algorithm $\cA$ is $\mathfrak{D}$-universally consistent with rate $\xi(n)$
and the problem itself is learnable in a distribution-free sense with rate $\xi'(n)$,
then there exists a $\mathfrak{D}$-universally consistent learning algorithm $\cA'$ that is $\mathfrak{D}$-universally AERM with rate $12\xi'(n^{1/4}) + \frac{37}{\sqrt{n}} + \xi(\sqrt{n})$ and satisfies $\frac{2}{\sqrt{n}}(e^{\epsilon}-e^{-\epsilon})$-DP.
\end{lemma}
The proof, given in \Cref{sec:proof-D-learnability}, shows that the algorithm $\cA'$ that applies $\cA$ to a random subsample of size $\lfloor\sqrt{n}\rfloor$ is AERM for any distribution in the class $\mathfrak{D}$.

\begin{theorem}[Characterization of private $\mathfrak D$-learnability]\label{thm:characterization_mod}
A problem is privately $\mathfrak{D}$-learnable \emph{if} there exists an algorithm that is $\mathfrak{D}$-universally AERM and differentially private with privacy loss $\epsilon(n)\rightarrow 0$.
If in addition, the problem is (distribution-free and non-privately) learnable, then the converse is also true.
\end{theorem}
\begin{proof}
The ``if'' part is exactly the same as the argument in Section~\ref{sec:privacy_stability}, since both Lemma~\ref{lem:EM_stability} and Lemma~\ref{corr:consistency} holds for each distribution independently.
Under the additional assumption that the problem itself is learnable (distribution-free and non-privately), the ``only if'' part is given by Lemma~\ref{lem:necessity_mod}.
\end{proof}

	This result may appear to be unsatisfactory due to the additional assumption of learnability. It is clearly a strong assumption because many problems that are $\mathfrak{D}$-learnable for a practically meaningful $\mathfrak{D}$ are not actually learnable. We provide one such example here.
\\	
	\begin{example}\label{ex:estimate_discrete}
			Let the data space be $[0,1]$, the hypothesis space be the class of all \emph{finite} subset of $[0,1]$ and the loss function $\ell(h,z) = 1_{z\notin h}$. This problem is not learnable, and not even $\mathfrak{D}$-learnable when $\mathfrak{D}$ is the class of all discrete distributions with finite number of possible values. But it is $\mathfrak{D}$-learnable when $\mathfrak{D}$ is further restricted with an upper bound on the total number of possible values. 
	\end{example}
	\begin{proof}
		For any discrete distribution with a finite support set, there is an $h\in\cH$ such that the optimal risk is $0$. Assume the problem is learnable with rate $\xi(n)$, then for some $n$ $\xi(n)<0.5$. However, we can always construct a uniform distribution over $3n$ elements and it is information-theoretically impossible for any estimators based on $n$ samples from the distribution to achieve a risk better than $2/3$. The problem is therefore not learnable. When we assume an upper bound $N$ on the maximum number of bins of the underlying distribution, then the ERM which outputs just the support of all observed data will be universally consistent with rate $\xi(n) = N/n$. 
	\end{proof}
	It turns out that we cannot hope to \emph{completely} remove the assumption from Theorem~\ref{thm:characterization_mod}. The following example illustrates that some form of qualification (implied by the learnability assumption) is necessary for the converse statement to be true. 
\\	
\begin{example}\label{ex:converse}
	Consider the learning problem in \Cref{ex:estimate_discrete}. Let $\mathfrak{D}$ be the class of all continuous distributions. There is a learning problem that is s privately $\mathfrak{D}$-learnable but no private AERM algorithm exists.
\end{example}	
	\begin{proof}
		Let the learning problem be that in Example~\ref{ex:estimate_discrete} and $\mathfrak{D}$ be the class of all continuous distributions defined on $[0,1]$.
		Consider The learning algorithm $\cA(Z)$ always returns $h = \emptyset$.\\	
	The optimal risk for any continuous distribution is $1$ because any finite subset is of measure $0$, output $\emptyset$ is $0$-consistent and $0$-generalizing, but not AERM, since the minimum empirical risk is $0$. $\cA$ is also $0$-differentially private, therefore the problem is privately $\mathfrak{D}$-learnable for $\mathfrak{D}$ being the set of all continuous distributions.\\	
	However, it is not privately $\mathfrak{D}$-learnable via an AERM, i.e., no private AERM algorithm exists for this problem.
	 We prove this by contradiction. Assume an $\epsilon$-DP AERM algorithm exists, the subsampling lemma ensures the existence of an $\epsilon(n)$-DP AERM algorithm $\cA'$ with $\epsilon(n)\rightarrow 0$. $\cA'$ is therefore generalizing by stability, and it follows that the $\cA'$ has risk $\E_{h\sim\cA'(Z)}R(h)$ converging to $0$. But there is no $h\in\cH$ such that $R(h)<1$, giving the contradiction.
	 \end{proof}
Interestingly, this problem is $\mathfrak{D}$-learnable via a non-private AERM algorithm, which always outputs $h={Z}$. This is $0$-consistent, $0$-AERM but not generalizing. This example suggests that $\mathfrak{D}$-learnability and learnability are quite different because for learnable problems, if an algorithm is consistent and AERM, then it must also be generalizing \citep[Theorem~10]{shalev2010learnability}. 

\subsection{A generic learning algorithm}
The characterization of private learnability suggests a generic (but impractical) procedure that learns all privately learnable problems (in the same flavor as the generic algorithm in \citet{shalev2010learnability} that learns all learnable problems). This is to solve
\begin{equation}\label{eq:generic_all}
\argmin_{\parbox{0.7in}{\centering  \scriptsize{$(\mathcal A, \epsilon):$ \\ $\cA: \cZ^n \rightarrow \cH$,\\ $\cA$ is $\epsilon$-DP}}} \left[\epsilon +\sup_{Z\in \cZ^n} \left(\E_{h\sim \cA(Z)}\hat{R}(h,Z)  - \inf_{h\in\cH} \hat{R}(h,Z)\right)\right],
\end{equation}
or to privately $\mathfrak{D}$-learn the problem when \eqref{eq:generic_all} is not feasible
\begin{equation}\label{eq:generic_all_mod}
\argmin_{\parbox{0.7in}{\centering \scriptsize{$(\mathcal A, \epsilon):$ \\ $\cA: \cZ^n \rightarrow \cH$,\\ $\cA$ is $\epsilon$-DP}}} \left[\epsilon +\sup_{\cD\in \mathfrak{D}}\E_{Z\sim \cD^n} \left(\E_{h\sim \cA(Z)}\hat{R}(h,Z)  - \inf_{h\in\cH} \hat{R}(h,Z)\right)\right].
\end{equation}
\begin{theorem}
Assume the problem is learnable.
If the problem is private learnable, \eqref{eq:generic_all} will always output a universally consistent private learning algorithm. If the problem is private $\mathfrak{D}$-learnable, \eqref{eq:generic_all_mod} will always output a $\mathfrak{D}$-universally consistent private learning algorithm.
\end{theorem}
\begin{proof}
If the problem is private learnable, by Theorem~\ref{thm:characterization} there exists an algorithm $\cA$ that is $\epsilon(n)$-DP and  always AERM with rate $\xi(n)$ and $\epsilon(n)+\xi(n) \rightarrow 0$. This $\cA$ is a witness in the optimization so we know that any minimizer of \eqref{eq:generic_all} will have a objective value that is no greater than $\epsilon(n)+\xi(n)$ for any $n$. Corollary~\ref{corr:consistency} concludes its universal consistency.
The second claim follows from the characterization of private $\mathfrak{D}$-learnability in Theorem~\ref{thm:characterization_mod}.
\end{proof}



It is of course impossible to minimize the supremum over any data $Z$, nor is it possible to efficiently search over the space of all algorithms, let alone DP algorithms. But conceptually, this formulation may be of interest to theoretical questions related to the search of private learning algorithms and the fundamental limit of machine learning under privacy constraints.

\section{Private learning for penalized ERM}\label{sec:pen-erm}
Now we describe a generic and practical class of private learning algorithms,
based on the idea of minimizing the empirical risk under privacy constraint:
\begin{equation}\label{eq:generic_optimization}
  \minimize_{h\in\cH} F(Z,h)=\frac{1}{n}\sum_{i=1}^{n} \ell(h,z_i) + g_n(h).
\end{equation}
The first term is empirical risk and the second term vanishes as $n$ increases so that this estimator is asymptotically ERM. The same formulation has been studied before in the context of differentially private machine learning \citep{chaudhuri2011differentially,kifer2012private}, but our focus is more generic and does not require the objective function to be convex, differentiable, continuous, or even have a finite dimensional Euclidean space embedding, hence covers a larger class of learning problems.


%

Our generic algorithm for differentially private learning is summarized in Algorithm~\ref{alg:EM_erm}. It applies the exponential mechanism \citep{mcsherry2007mechanism}
to penalized ERM. We note that this algorithm implicitly requires that $\int_{\cH} \exp(\frac{\epsilon(n)}{2\Delta q}q(h,Z)) dh < \infty$, otherwise the distribution is not well-defined and it does not make sense to talk about differential privacy. In general, if $\cH$ is a compact set with a finite volume (with respect to a base measure, such as the Lebesgue measure or counting measure), then such  a distribution always exists. We will revisit this point and discuss the practicality of this assumption in the Section~\ref{sec:inf_domain}.
\begin{algorithm}[t]
   \caption{ Exponential Mechanism for regularized ERM}
   \label{alg:EM_erm}
\begin{algorithmic}
   \STATE {\bfseries Input:} Data points $Z=\{z_1,...,z_n\}\in \cZ^n$, loss function $\ell$, regularizer $g_n$, privacy parameter $\epsilon(n)$ and a hypothesis space $\cH$.
   \STATE {1.} Construct utility function $q(h,Z):=-\frac{1}{n}\sum_{i=1}^{n}\ell(h,z_i) - g_n(h)$, and its sensitivity $\Delta q:=\sup_{h\in \cH,d(Z,Z')=1}|q(h,Z)-q(h,Z')|\le \frac{2}{n}\sup_{h\in\cH,z\in\cZ}\left|\ell(h,z)\right|.$
    \STATE{2.} Sample $h\in \cH$ with probability $\P(h)\propto\exp(\frac{\epsilon(n)}{2\Delta q} q(h,Z))$.
   \STATE {\bfseries Output:} $h$.
\end{algorithmic}
\end{algorithm}

Using the characterization results developed so far, we are able to give sufficient conditions for consistency of private learning algorithms without having to establish uniform convergence.  Define the sublevel set as
   \begin{equation}\label{eq:sublevel_set}
     \cS_{Z,t} = \{ h\in \cH \;|\; F(Z,h)\leq t+ \inf_{h\in\cH}F(Z,h)\}\,,
   \end{equation}
   where $F(h,Z)$ is the regularized empirical risk function defined in (\ref{eq:generic_optimization}).
  In particular, we assume the following conditions:

\textbf{A1}. Bounded loss function:
  $0\leq \ell(h,z)  \leq 1 \; \text{ for any }\; h\in \cH, z\in \cZ.$

\textbf{A2}. Sublevel set condition: There exist constant positive integer $n_0$, positive real number $t_0$, and a sequence of regularizer $g_n$ satisfying $\sup_{h\in \cH}|g_n(h)|=o(n)$, such that for any $0<t<t_0$, $n>n_0$
     \begin{equation}\label{eq:measure_cond}
     \E_{Z\sim \cD^n}\left(\frac{\mu(\cH)}{\mu(\cS_{Z,t})}\right) \leq K\left(\frac{1}{t}\right)^\rho,
  \end{equation}
  where $K=K(n),\rho=\rho(n)$ satisfy $\log K + \rho\log n=o(n)$.
Here the measure $\mu$ may depend on context, such as Lebesgue measure ($\cH$ is continuous) or counting measure ($\cH$ is discrete).



The first condition of boundedness is common. It is assumed in \citeauthor{vapnik1998statistical}'s characterization for ERM learnability and \citeauthor{shalev2010learnability}'s general characterization of all learnable problems. In fact, we can always consider $\cH$ to be a sublevel set such that the boundedness condition holds. For the second condition, the intuition is that we require the sublevel set to be large enough such that the sampling procedure will return a good hypothesis with large probability. $\mu(\cS_t)$ is a critical parameter in the utility guarantee for the exponential mechanism \citep{mcsherry2007mechanism}.  Also, it is worth pointing out that A2 implies that the exponential distribution is well-defined.
\begin{theorem}[General private learning]\label{thm:general_private_learning}
Let $(\cZ,\cH,\ell)$ be any problem in the general learning setting. Suppose we can choose $g_n$  such that A.1 and A.2 are satisfied with $(\rho,K,g_n,n_0, t_0)$ for a distribution $\cD$, then Algorithm \ref{alg:EM_erm} satisfies $\epsilon(n)$-privacy and  is consistent with rate
\begin{equation}\label{eq:thm_cons}
  \xi(n) = \frac{9[\log K+(\rho+2)\log n]}{n\epsilon(n)}+2\epsilon(n) + \sup_{h\in\cH}|g_n(h)|.
\end{equation}
In particular, if $\epsilon(n)=o(1)$, $\sup_{h\in\cH}|g_n(h)|=o(1)$ and $\log K+\rho\log n=o(n\epsilon(n))$ for all $\cD$ (in $\mathfrak{D}$)
Algorithm \ref{alg:EM_erm} privately learns ($\mathfrak{D}$-learns) the problem.
\end{theorem}
We give an illustration of the proof in Figure~\ref{fig:bigpic2}. The detailed proof,  based on the stability argument~\citep{shalev2010learnability}, is deferred to \Cref{sec:private_beyond_uniform_convergence}. 

\begin{figure}[t]
  \centering
  \includegraphics[width=0.8\textwidth]{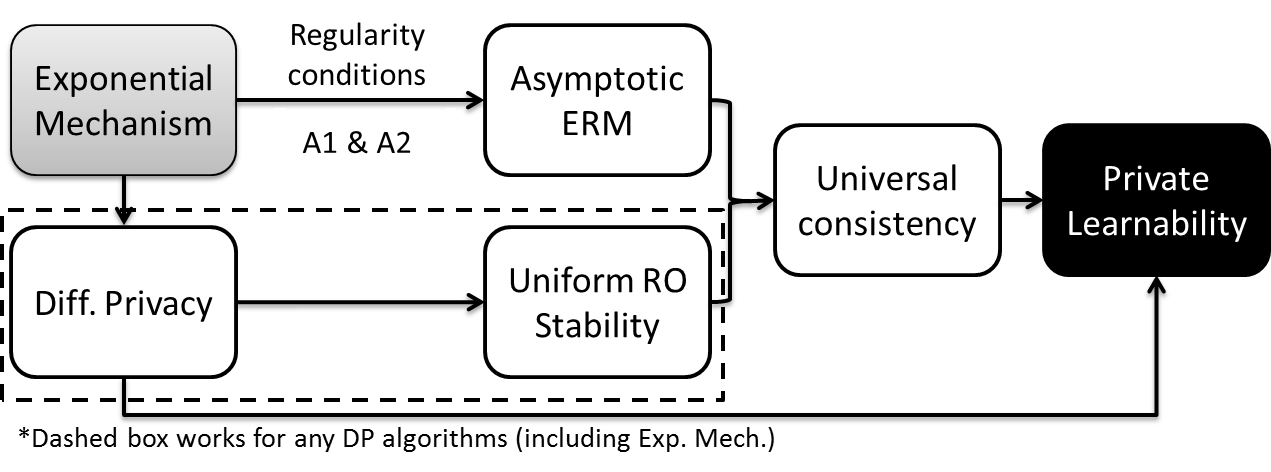}\\
  \caption{Illustration of Theorem~\ref{thm:general_private_learning}: conditions for private learnability in general learning setting.}\label{fig:bigpic2}
\end{figure}

To see that Theorem~\ref{thm:general_private_learning} actually contains a large number of problems in the general learning setting. We provide concrete examples that satisfy A1 and A2 below for both privately learnable and privately $\mathfrak{D}$-learnable problems that can be learned using Algorithm~\ref{alg:EM_erm}.

\subsection{Examples of privately learnable problems}
We start from a few cases where Algorithm~\ref{alg:EM_erm} is universally consistent for all distributions.
\begin{example}[Finite discrete $\cH$]\label{exmp:discrete}
Suppose $\cH$ can be fully encoded by $M$-bits, then
$$
\mu(\cS_t)/\mu(\cH)\geq |\cH|^{-1} = 2^{-M}\,,
$$
since there are at least $1$ optimal hypothesis for each function and now $\mu$ is the counting measure. In other word, we can take $K=2^{M}$ and $\rho=0$ in the \eqref{eq:thm_cons}. Plug this into the expression and take $g_n\equiv 0$, $\epsilon(n)=\sqrt{(M+\log n)/n}$, we get a rate of consistency $\xi(n)=O(\frac{M+\log n}{\sqrt{n}})$.
In addition, if we can find a data-independent covering set for a continuous space, then we can discretize the space and the result same results follow.  This observation
will be used in the construction of many private learning algorithms below.

\end{example}
\begin{example}[Lipschitz functions/H\"{o}lder class]\label{exmp:Lipschitz}
Let $\cH$ be a compact,  $\beta_p$-regular subset of $\mathbb R^d$ satisfying
$\mu(B\cap \cH)\ge \beta_p \mu(B)$ for any $\ell_p$ ball $B\subset \mathbb R^d$ that is small enough. Assume that $F(Z,\cdot)$ is $L$-Lipschitz on $\cH$: for any $h,h^{\prime}\in \cH$,
$$|F(Z,h)-F(Z,h^\prime)|\leq  L\|h-h^\prime\|_p\,.$$
Then for sufficiently small $t$, we have Lebesgue measure
$$
\mu(\cS_t) \geq
\beta_p\left(t/L\right)^d\,
$$
and Condition A.2 holds with $K=\mu(\cH)\beta_p^{-1}L^{d}$, $\rho=d$.
Furthermore, if we take $\epsilon(n)=\sqrt{\frac{d(\log L+\log n)+\log(\mu(\cH)/\beta_p)}{n}}$, the algorithm is  $O\left(\sqrt{\frac{d(\log L+\log n)+\log(\mu(\cH)/\beta_p)}{n}} + \underset{h\in\cH}{\sup}|g_n(h)|\right)$-consistent. \end{example}

This shows that condition \textrm{A2} holds for a large class of low-dimensional problems of interest in machine learning and one can learn the problem privately without actually needing to find a covering set algorithmically.
Specifically, the example includes many practically used methods such as logistic regression, linear SVM, ridge regression, even multi-layer neural networks, since the loss functions in these methods are jointly bounded in $(Z, h)$  and Lipschitz in $h$.

The example also raises an interesting observation that while differentially private classification is not possible in a distribution-free setting for 0-1 loss function \citep{chaudhuri2011sample}, it is learnable under smoother surrogate loss, e.g., logistic loss or hinge loss.
In other words, private learnability and computational tractability both benefit from the same relaxation.



The Lipschitz condition still requires the dimension of the hypothesis space to be $o(n)$.  Thus it does not cover high-dimensional machine learning problems where $d\gg n$, nor does it contain the example of \citet{shalev2010learnability} that ERM fails.

For high dimensional problems where $d$ grows with $n$, typically some assumptions or restrictions need to be made either on the data or on the hypothesis space (so that it becomes essentially low-dimensional). We give one example here for the problem of sparse regression.





\begin{example}[Best subset selection]
 Consider $\cH = \{h\in \R^{d}: \|h\|_0<s, \|h\|_2\leq 1\}$ and let $\ell(h,z)$ be an $L$-Lipschitz loss function. The solution can only be chosen from ${d\choose s} < d^s$ different $s$-dimensional subspaces. We can apply Algorithm~\ref{alg:EM_erm} twice to first sample a support set $S$ with utility function being the $-\min_{h\in \cH_S}F(Z,h)$, and then sample a solution in the chosen $s$-dimensional subspace. By the composition theorem this two-stage procedure is differentially private. Moreover, by the arguments in Example~\ref{exmp:discrete} and Example~\ref{exmp:Lipschitz} respectively, we have an $\mu(\cS_t)\geq \left(\frac{1}{d}\right)^s$ for the subset selection and $\mu(\cS_t)\geq (\frac{t}{L})^s$ for the low-dimensional regression.  Note that $\rho=0$ in both cases and the dependency on the ambient dimension $d$ is on the logarithm.  The first stage ensures that for the chosen support set $\cS$, $\min_{h\in \cH_S}F(Z,h)$ is close to $\min_{h\in \cH}F(Z,h)$ by $O(\frac{s\log d + \log n}{n\epsilon(n)})$ in expectation and ( the second stage ensures that the sampled hypothesis from $\cH_S$ would have objective function close to $\min_{h\in \cH_S}F(Z,h)$ by $O(\frac{s\log L+s\log n + \log(\mu(\cH_S)/\beta_p)}{n\epsilon(n)})$. This leads to an overall rate of consistency (they simply add up) of $O(\frac{s(\log d+\log n +L)+\log(\mu(\cH_S)/\beta_p)}{\sqrt{n}})$ if we choose $\epsilon(n)=1/\sqrt{n}$.
\end{example}

\subsection{Examples of privately $\mathfrak{D}$-learnable problems.}
For problems where private learnability is impossible to achieve, we may still apply  Theorem~\ref{thm:general_private_learning} to prove the weaker private $\mathfrak{D}$-learnability for some specific class of distributions.

\begin{example}[Finite Representation Dimension in the General Learning Setting]
For binary classification problems with $0$-$1$ loss (PAC learning), this has been well-studied. In particular, \citet{beimel2013characterizing} characterized the sample complexity of privately learnable problems using a combinatorial condition they call a ``Probabilistic Representation'', which basically involves finding a finite, data-independent set of  hypotheses to approximate any hypothesis in the class. Their claim is that if the ``representation dimension'' is finite, then the problem is privately learnable, otherwise it is not.
We can extend the notion of probabilistic representation beyond the finite discrete and countably infinite hypothesis class considered in \citet{beimel2013characterizing} to cases when the problem is not privately learnable (e.g, learning threshold functions on $[0,1]$). The existence of probabilistic representation for all distributions in $\mathfrak{D}$ would lead to a $\mathfrak{D}$-universally private learning algorithm.
\end{example}


Another way to define a class of distribution $\mathfrak{D}$ is to assume the existence of a reference distribution that is close to any distribution of interest as in \cite{chaudhuri2011sample}.
\begin{example}[Existence of a public reference distribution]
To deal with the $0$-$1$ loss classification problems on a continuous hypothesis domain, \citet{chaudhuri2011sample} assume that there exists a data-independent reference distribution $\cD^*$, which by multiplying a fixed constant on its density, uniformly dominates any distributtion of interest. This essentially produces a subset of distributions $\mathfrak{D}$. The consequence is that one can build an $\epsilon$-net of $\cH$ with metric defined on the risk under $\cD^*$ and this will also be a (looser) covering set of any distribution $\cD\in\mathfrak{D}$, thereby learning the problem for any distribution in the set.

The same idea can be applied to the general learning setting. For any fixed reference distribution $\cD^*$ defined on $\cZ$ and constant $c$,
$$\mathfrak{D}=\{\cD=(\cZ,\cF,\P) \mid  \P_\cD(z\in A)\leq c \P_{\cD^*}(z\in A)\text{ for }\forall A \in \cF\}$$
is a valid set of distributions and we are able to $\mathfrak{D}$-privately learn this problem whenever
we can construct a sufficiently small cover set with respect to $\cD^*$ and reduce the problem to Example~\ref{exmp:discrete}.  This class of problems includes high-dimensional and infinity dimensional problems such as density estimation, nonparametric regression, kernel methods and essentially any other problems that are strictly learnable \citep{vapnik1998statistical}, since they are characterized by one-sided uniform convergence (and the corresponding entropy condition).


\end{example}

\subsection{Discussion on uniform convergence and private learnability}

Uniform convergence requires that $\E_{Z\sim \cD^n} \sup_{h\in\cH} |\hat{R}(h,Z) - R(h)| \rightarrow 0$ for any distribution $\cD$ with a distribution independent rate. Most machine learning algorithms rely on uniform convergence to establish consistency result (e.g., through complexity measure such as VC-dimension, Rademacher Complexity, covering and bracketing numbers and so on). In fact, the learnability of ERM algorithm is characterized by the one-sided uniform convergence \citep{vapnik1998statistical}, which is only slightly weaker than requiring uniform convergence on both sides.

A key point in \citet{shalev2010learnability} is that the learnability (by any algorithm) in general learning setting is no longer characterized by variants of uniform convergence. However, the class of privately learnable problems is much smaller. Clearly, uniform convergence is not sufficient for a problem to be privately learnable (see Section~\ref{sec:counter_example}), but is it necessary? 

In binary classification with discrete domain (agnostic PAC Learning), since VC-dimension being finite characterizes the class of privately PAC learnable problems, the necessity of uniform convergence is clear. This could also be more explicitly seen from \citet{beimel2013characterizing} where the \emph{probabilistic representation dimension} is a form of uniform convergence on its own.

In the general learning setting, the problem is still open. We were not able to prove that private learnability implies uniform convergence, but we could not construct a counter example either. All our examples in this section do implicitly or explicitly uses uniform convergence, which seems to hint at a positive answer.

\section{Practical concerns}
\subsection{High confidence private learning via boosting}
We have stated all results so far  in expectation. We can easily convert these to the high-confidence learning paradigm by applying Markov's inequality, since convergence in expectation to the minimum risk implies convergence in probability to the minimum risk. While the $1/\delta$ dependence on the failure probability $\delta$ is not ideal, we can apply a similar meta-algorithm ``boosting''\citep{schapire1990strength} as in \citet[Section~7]{shalev2010learnability} to get a $\log(1/\delta)$ rate.
The approach is  similar to cross-validation. Given a pre-chosen positive integer $a$, the original boosting algorithm randomly partitions the data into $(a+1)$ subsamples of size $n/(a+1)$, and applies Algorithm~\ref{alg:EM_erm} on the first $a$ partitions, obtaining $a$ candidate hypotheses.  
The method then returns the one hypothesis with smallest validation error, calculated using the remaining subsample. To ensure differential privacy, our method instead uses the exponential mechanism to sample the best candidate hypothesis, where the logarithm of sampling probability is proportional to the negative validation error.
\begin{theorem}[High-confidence private learning]\label{thm:high_confidence}
If an algorithm $\cA$ privately learns a problem with rate $\xi(n)$ and privacy parameter $\epsilon(n)$, then the  boosting algorithm $\cA^{\prime}$ with $a=\log\frac{3}{\delta}$ is
$\max\left\{\epsilon\left(\frac{n}{\log(3/n)+1}\right), \frac{\log(3/\delta)+1}{\sqrt{n}} \right\}$-differentially private, its output $h$ obeys
$$R(h) - R^* \leq  e\xi\left(\frac{n}{\log(3/\delta)+1}\right) + C\sqrt{\frac{\log(3/\delta)}{n}} $$
for an absolute constant $C$ with probability at least $1-\delta$. 
\end{theorem}

\subsection{Efficient sampling algorithm for convex problems}
Our proposed exponential sampling based algorithm is to establish a more explicit geometric condition upon which AERM holds, hence the algorithm may not be computationally tractable.
Ignoring the difficulty of constructing the $\epsilon$-covering set of an exponential number of elements, sampling from the set alone is not a polynomial time algorithm.  But we can solve  a subset of the continuous version of our Algorithm~\ref{alg:EM_erm} described in Theorem~\ref{thm:general_private_learning} in polynomial time to arbitrary accuracy (see also \citet[Theorem~3.4]{bassily2014private}).
\begin{proposition}\label{prop:efficiency}
If $n^{-1}\sum_{i=1}^{n}\ell(h,z_i)+g_n(h)$ is convex in $h$ and $\cH$ is a convex set, then the sampling procedure in Algorithm~\ref{alg:EM_erm} can be solved in polynomial time.
\end{proposition}
\begin{proof}
When $n^{-1}\sum_{i=1}^{n}\ell(h,z_i)+g_n(h)$ is convex, the utility function $q(h,Z)$ is concave in $h$. The density to be sampled from in Algorithm \ref{alg:EM_erm} is proportional to $\exp(\frac{\epsilon n q(h,Z)}{B})$ and is log-concave. The Markov chain sampling algorithm in \citet{applegate1991sampling} is guaranteed to produce a sample from a distribution that is arbitrarily close to the target distribution (in the total variation sense) in polynomial time.
\end{proof}

\subsection{Exponential mechanism in infinite domain}\label{sec:inf_domain}

As we mention earlier, the results in Section~\ref{sec:pen-erm} based on the exponential mechanism implicitly assumes certain regularity conditions that ensures the existence  of a probability distribution.

When $\cH$ is finite, the existence is trivial. On the other hand, an infinite set $\cH$ is tricky in that there may not exist a proper distribution that satisfies $\P(h) \propto e^{\frac{\epsilon}{2\Delta q}q(Z,h)}$ for at least some $q(Z,h)$. For instance, if $\cH = \R$ and $q(Z,h)\equiv 1$ then $\int_{\R}e^{\frac{\epsilon}{2\Delta q}q(Z,h)}dh = \infty$. Such distributions that are only defined up to scale with no finite normalization constants are called improper distributions. In case of finite dimensional non-compact set, this translates into an additional assumption on the loss function and the regularization term.

Things get even trickier when $\cH$ is an infinite dimensional space, such as a subset of a Hilbert space. While probability measures can still be defined, no density function can be defined on such spaces. Therefore, we cannot use exponential mechanism to define a valid probability distribution. 

The practical implication is that exponential mechanism is really only applicable to cases when the hypothesis space $\cH$ allows for definitions of densities in the usual sense, or then $\cH$ can be approximated by such a space. For example, a separable Hilbert space can be studied by finite-dimensional projections. Also, we can approximate RKHS induced by translation invariant kernels via random Fourier features \citep{rahimi2007random}.


\section{Results for learnability under $(\epsilon,\delta)$-differential privacy}\label{sec:eps_delta}
Another way to weaken the definition of private learnability is through $(\epsilon,\delta)$-approximate differential privacy. 
\begin{definition}[\citealp{dwork2006our}]
	An algorithm $\cA$ obeys $(\epsilon,\delta)$-differential privacy if for any $Z,Z'$ such that $d(Z,Z')\leq 1$, and for any measurable set $\cS\subset \cH$
	$$
	\P_{h\sim \cA(Z)}(h\in\cS) \leq e^\epsilon \P_{h\sim \cA(Z')}(h\in\cS) + \delta.
	$$
\end{definition}
We define a version of the problem to be
\begin{definition}[Approximately Private Learnability]
	We say a learning problem is $\Delta(n)$-approximately privately learnable for some pre-specified family of rate $\Delta(n)$ if for some $\epsilon<\infty$, $\delta(n)\in \Delta(n)$, there exists a universally consistent algorithm that is $(\epsilon,\delta(n))$-DP.
\end{definition}
This is a completely different subject to study and the class of approximately privately learnable problems could be substantially larger than the pure privately learnable problems. Moreover, the picture may vary with respect to how small $\delta(n)$ is required to be. In this section, we present our preliminary investigation on this problem.

Specifically, we will consider two questions:
\begin{enumerate}
	\item Does the existence of an $(\epsilon,\delta)$-DP always AERM algorithm characterize the class of approximately private learnable problems?
	\item Are all learnable problems approximately privately learnable for different choices of $\Delta(n)$?
\end{enumerate}

The minimal requirement in the same flavor of Definition~\ref{def:private-learnability} would be to require $\Delta(n)=\{\delta(n)|\delta(n) \rightarrow 0\}$. The learnability problem turns out to be trivial under this definition due to the following observation.
\begin{lemma}\label{lem:subsampling_eps_delta}
	For any algorithm $\cA$ that acts on $Z$, $\cA'$ that runs $\cA$ on a randomly chosen subset of $Z$ of size $\sqrt{n}$ is $(0,\frac{1}{\sqrt{n}})$-DP.
\end{lemma}
\begin{proof}
	Let $Z$ and $Z'$ be adjacent datasets that differs only in data point $i$. For any $i$ and any $S \in \sigma(\cH)$.
	\begin{align*}
	\P(\cA'(Z) \in S)  &= \P_I(\cA(Z_I)\in S|i\in I) \P(i\in I) + \P_I(\cA(Z_I)\in S|i\notin I) \P(i\notin I)\\
	&=\P_I(\cA(Z_I)\in S|i\in I) \P(i\in I) + \P_I(\cA(Z'_I)\in S|i\notin I) \P(i\notin I)\\
	&= \P(\cA'(Z') \in S)  + [\P_I(\cA(Z_I)\in S|i\in I)- \P_I(\cA'(Z_I)\in S|i\in I) ] \P(i\in I)\\
	&\leq \P(\cA'(Z') \in S) + \P(i\in I)\\
	&= e^0 \P(\cA'(Z') \in S)  + \frac{1}{\sqrt{n}}.
	\end{align*}
	This verifies the $(0,1/\sqrt{n})$-DP of algorithm $\cA'$.
\end{proof}
The above lemma suggests that if $\delta(n)=  o(1)$ is all we need for the \emph{approximately private learnability}, then any consistent learning algorithm can be made approximately DP by simply subsampling. In other words, any learnable problem is also learnable under approximate differential privacy.

To get around this triviality, we need to specify a sufficiently fast rate of $\delta(n)$ going to $0$. While it is common to require that $\delta(n)  = o(1/\text{poly}(n))$\ \footnote{Here the notation ``$o(1/\text{poly}(n))$'' means ``decays faster than any polynomial of $n$''.  A sequence $a(n)=o(1/\text{poly}(n))$ if and only if $a(n)=o(n^{-r})$ for any $r>0$.} for cryptographically strong privacy protection, requiring $\delta(n) = o(1/n)$ is already enough to invalidate the above subsampling argument and makes the problem of learnability a non-trivial one.

Again, the question is whether AERM characterizes approximately private learnability and whether there is a gap between the class of learnable and approximately privately learnable problems.

Here we show that the ``folklore'' Lemma~\ref{lem:EM_stability} and subsampling lemma (Lemma~\ref{lem:sampling_thm}) can be extended to work with $(\epsilon,\delta)$-DP and then we provide a positive answer to the first question.
\begin{lemma}[Stability of $(\epsilon,\delta)$-DP]\label{lem:stability_approx}
	If $\cA$ is $(\epsilon,\delta)$-DP, and $0\leq \ell(h,z)\leq 1$, then $\cA$ is $(e^\epsilon - 1 + \delta)$-Strongly Uniform RO-stable.
\end{lemma}
\begin{proof}
	For any $Z,Z'$ such that $d(Z,Z')\leq 1$ and for any $z\in \cZ$. Let the event $E = \{h | p(h)\geq p'(h)\}$, 
	\begin{align*}
	&\left|\E_{h\sim\cA(Z)} \ell(h,z) - \E_{h\sim\cA(Z')}\ell(h,z)\right| = \left|\int_h  \ell(h,z)p(h)dh - \int_h\ell(h,z)p'(h) dh\right| \\
	\leq& \sup_{h,z}\ell(h,z)\int_E  p(h) - p'(h)dh \leq  \int_E  p(h) - p'(h)dh  
	= \P_{h\sim \cA(Z)}(h\in E) - \P_{h\sim \cA(Z')}(h\in E)\\
	\leq& (e^\epsilon-1)\P_{h\sim \cA(Z')}(h\in E)  + \delta \leq e^\epsilon-1  + \delta.
	\end{align*}
	The last line applies the definition of $(\epsilon,\delta)$-DP.
\end{proof}
\begin{lemma}[Subsampling Lemma of $(\epsilon,\delta)$-DP]\label{lem:subsampling_approx}
	If $\cA$ is $(\epsilon,\delta)$-DP, then $\cA'$ that acts on a random subsample of $Z$ of size $\gamma n$ obeys $(\epsilon',\delta')$-DP with $\epsilon' = \log(1+\gamma e^\epsilon(e^{\epsilon}-1))$ and $\delta' = \gamma e^\epsilon\delta$.
\end{lemma}
\begin{proof}
	For any event $E\in \sigma(\cH)$, let $i$ be the coordinate where $Z$ and $Z'$ differs
\begin{align}
	&\P_{h\sim \cA'(Z)}(h\in E) = \gamma \P_{h\sim A(Z_I)}(h\sim E | i\in I ) + (1-\gamma)\P_{h\sim A(Z_I)}(h\sim E | i\notin I )\nonumber\\
	=&\gamma \P_{h\sim A(Z_I)}(h\sim E | i\in I ) + (1-\gamma)\P_{h\sim A(Z'_I)}(h\sim E | i\notin I )\nonumber\\
	=& \gamma \P_{h\sim A(Z_I)}(h\sim E | i\in I ) -\gamma \P_{h\sim A(Z'_I)}(h\sim E | i\in I ) + \gamma \P_{h\sim A(Z'_I)}(h\sim E | i\in I ) \nonumber\\
	&+(1-\gamma)\P_{h\sim A(Z'_I)}(h\sim E | i\notin I )\nonumber\\
	=& \P_{h\sim \cA'(Z')}(h\in E) + \gamma [\P_{h\sim A(Z_I)}(h\sim E | i\in I ) -\P_{h\sim A(Z'_I)}(h\sim E | i\in I ) ]\nonumber\\
	\leq& \P_{h\sim \cA'(Z')}(h\in E) + \gamma (e^\epsilon-1)\P_{h\sim A(Z'_I)}(h\sim E | i\in I ) + \gamma \delta,\label{eq:subsampling_eps_delta_der1}
\end{align}
where in last line, we apply $(\epsilon,\delta)$-DP of $\cA$. 

It remains to show that $\P_{h\sim A(Z'_I)}(h\sim E | i\in I )$ is similar to $\P_{h\sim \cA'(Z')}(h\in E)$. First,
\begin{equation}
\P_{h\sim \cA'(Z')}(h\in E)  = \gamma \P_{h\sim \cA(Z'_{I})}(h\in E|i\in I) + (1-\gamma )\P_{h\sim \cA(Z'_{I})}(h\in E|i\notin I).\label{eq:subsampling_eps_delta_der2}
\end{equation}
Denote $\cI_1 =\{ I | i\in I\}$, $\cI_2 =\{ I | i\notin I\}$. We known $|\cI_1| = {n-1\choose\gamma n -1}$, and $|\cI_2| = {n-1\choose\gamma n}$ and $|\cI_1|/|\cI_2| = \gamma n/(n-\gamma n)$. For every $I \in \cI_2$ there are precisely $\gamma n$ elements $J\in\cI_1$ such that $d(I,J) = 1$. Likewise, for every $J\in \cI_1$, there are $n-\gamma n$ elements $I\in \cI_2$ such that $d(I,J)=1$.
It follows by symmetry that if we apply $(\epsilon,\delta)$-DP to $1/\gamma n$ of each $I \in \cI_2$ and change $I$ to their corresponding $J\in \cI_1$,  then each $J\in \cI_1$ will receive $(n-\gamma n)/\gamma n$ ``contribution'' in total from the sum over all $I \in \cI_2$.
\begin{align*}
&\P_{h\sim \cA(Z'_{I})}(h\in E|i\notin I) = \frac{1}{|\cI_2|}\sum_{I\in \cI_2} \P_{h\sim \cA(Z'_{I})}(h\in E) \\
=& \frac{1}{|\cI_2|}\sum_{I\in \cI_2} \sum_{j=1}^{\gamma n} \frac{1}{\gamma n}\P_{h\sim \cA(Z'_{I})}(h\in E)\\
\geq& \frac{|\cI_1|}{|\cI_2|}  \frac{1}{|\cI_1|}\sum_{J \in \cI_1} \frac{n-\gamma n}{\gamma n } e^{-\epsilon}( \P_{h\sim \cA(Z'_{J})}(h\in E) -\delta)\\
=&\frac{1}{|\cI_1|}\sum_{J \in \cI_1} e^{-\epsilon}( \P_{h\sim \cA(Z'_{J})}(h\in E) -\delta) =  e^{-\epsilon}\P_{h\sim \cA(Z'_{I})}(h\in E|i \in I)  - e^{-\epsilon}\delta，
\end{align*}
Substitute into \eqref{eq:subsampling_eps_delta_der2}, we get 
$$
\P_{h\sim \cA(Z'_{I})}(h\in E|i \in I)  \leq  \frac{1}{\gamma + (1-\gamma)e^{-\epsilon}}  \P_{h\sim \cA'(Z')}(h\in E) + \frac{(1-\gamma)e^{-\epsilon}}{\gamma + (1-\gamma)e^{-\epsilon}} \delta.
$$
We further relax the upper bound to a simple form $e^\epsilon \P_{h\sim \cA'(Z')}(h\in E) + \delta$ and substitute into \eqref{eq:subsampling_eps_delta_der1}, we have
$$
\P_{h\sim \cA'(Z)}(h\in E) \leq (1+ \gamma e^\epsilon(e^\epsilon-1))\P_{h\sim \cA'(Z')}(h\in E) + \gamma \delta + \gamma(e^\epsilon-1)\delta,
$$
which concludes the proof.
\end{proof}

Using the above two lemmas, we are able to establish the same result which says that AERM characterizes the approximate private learnability for certain classes of $\Delta(n)$.
\begin{theorem}\label{thm:eps_delta_DP}
	A problem is $\Delta(n)$-approximately privately learnable
	implies that there exists an always AERM algorithm that is $(\epsilon(n),n^{-1/2}e^\epsilon\delta(\sqrt{n}))$-DP for some $\epsilon(n)\rightarrow 0$ and $\delta(\sqrt{n})\in \Delta(n)$. The converse is also true if $n^{-1/2}e^\epsilon\delta(\sqrt{n})\in \Delta(n)$.
\end{theorem}
\begin{proof}
	If we have an always AERM algorithm with $\xi_{erm}(n)$ that is $(\epsilon(n),\delta(n))$-DP for $\delta(n)\in \Delta(n)$. Then by Lemma~\ref{lem:stability_approx}, this algorithm is strongly uniform RO-stable with rate $e^{\epsilon(n)}-1 + \delta(n)$. By Theorem~\ref{thm:stable_aerm}, the algorithm is universally consistent with rate $\xi_{erm}(n) + e^{\epsilon(n)}-1 + \delta(n)$. This establishes the ``if'' part.
	
	To see the ``only if'' part, by definition if a problem is $\Delta(n)$-approximately privately learnable with $\epsilon$ and $\delta(n)\in \Delta(n)$. Then by Lemma~\ref{lem:subsampling_approx} with $\gamma = 1/\sqrt{n}$, we get an algorithm that obeys the privacy condition. It remains to prove always AERM, which requires exactly the same arguments in the proof of Lemma~\ref{lem:necessity}. Details are omitted.
\end{proof}
Note that the results above suggest that in the two canonical settings $\Delta(n)=o(1/n)$ or $\Delta(n)=o(1/\text{poly}(n))$, existence of a private AERM algorithm that satisfies the stronger constraint $\epsilon(n)=o(1)$ characterizes the learnability.

The next question that whether any learnable problems are also approximately privately learnable would depend on how fast $\delta(n)$ is required to decay. We know that when we only have $\Delta(n)=o(1)$, all learnable problems are approximately privately learnable, and when we have $\Delta(n)=\{0\}$, only a strict subset of these problems is privately learnable. The following result establishes that when $\delta(n)$ needs to go to $0$ with a sufficiently fast rate, there is separation between learnability and approximately private learnability.
\begin{proposition}\label{prop:phase_transition}
	Let $\Delta(n) =\{\delta(n)|\delta(n)\leq \tilde{\delta}(n)\}$ for some sequence $\tilde{\delta}(n)\rightarrow 0$. The following statements are true.
	\begin{itemize}
		\item  All learnable problems are $\Delta(n)$-approximately privately learnable, if $\tilde{\delta}(n)=\omega(1/n)$.
		\item There exists a problem that is learnable but not $\Delta(n)$-approximately privately learnable, if  $\tilde{\delta}(n) \leq \frac{\exp(-\epsilon(n^2) n^2)}{n}$
	\end{itemize}	
\end{proposition}
\begin{proof}
	The first claim follows from the same argument in Lemma~\ref{lem:subsampling_eps_delta}. If a problem is learnable, there exists a universally consistent learning algorithm $\cA$. The algorithm that applies $\cA$ on a $\tilde{\delta}(n)$-fraction random subsample of the dataset is $(0,\tilde{\delta}(n))$-DP and universally consistent with rate $\xi(n\tilde{\delta}(n))$. Since $\tilde{\delta}(n)=\omega(1/n)$, $n\tilde{\delta}(n) \rightarrow \infty$.
	
	We now show that when we require a fast decaying $\delta(n)$, then suddenly the example in Section~\ref{sec:counter_example} due to \citet{chaudhuri2011sample} becomes not approximately privately learnable even for $(\epsilon,\delta)$-DP. Let $Z,Z'$ be two completely different data sets, by repeatedly applying the definition of $(\epsilon,\delta)$-DP, for any set $\cS\subset \cH$
	$$\P(\cA(Z)\in \cS) \leq e^{n\epsilon}\P(\cA(Z)\in \cS) + \sum_{i=1}^{n}e^{(i-1)\epsilon}\delta \leq e^{n\epsilon}\P(\cA(Z')\in \cS) + ne^{(n-1)\epsilon}\delta. $$
	When we shift the inequality around, we get 
	$$\P(\cA(Z')\in \cS) \leq e^{-n\epsilon}\P(\cA(Z')\in \cS) - e^{-\epsilon}n\delta.$$
	
	Consider the same example in Section~\ref{sec:counter_example} where we hope to learn a threshold on $[0,1]$. Assuming there exists an algorithm $\cA$ that is universally AERM and $(\epsilon(n),\delta(n))$-DP for $\epsilon(n)<\infty$ and $\delta(n) \leq 0.4 n e^{-\epsilon n}$.
	
	Everything up to \eqref{eq:counter_exp_keyeq1} remains exactly the same. Now, apply the above implication of $(\epsilon,\delta)$-DP, we can replace \eqref{eq:counter_exp_keyeq1} for each $i=2,...,K$, by 
	$$ \P(\cA(Z_1) \in [h_i-\eta/3,h_i+\eta/3]) \geq \exp(-\epsilon n)\P(\cA(Z_i) \in [h_i-\eta/3,h_i+\eta/3]) - n\delta(n).$$
	Then \eqref{eq:counter_exp_keyeq2} becomes
	\begin{equation*}
	\P (\cA(Z_1)\notin [h_1-\eta/3,h_1+\eta/3]) \geq K \exp(-\epsilon n) 0.9  - K e^{-\epsilon}n\delta(n) \geq 0.9 \geq 0.5,
	\end{equation*}
	where the last inequality follows by $K>\exp(\epsilon n)$ and $\delta(n)\leq 0.4 n e^{-\epsilon n}$. This yields the same contradiction to always AERM of $\cA$ on $Z_1$, which requires $\P(\cA(Z_1)\notin [h_1-\eta/3,h_1+\eta/3])<0.1$. Therefore, such AERM does not exist. By the contrapositive of Theorem~\ref{thm:eps_delta_DP}, the problem is not approximately privately learnable for $\tilde{\delta}(n) \leq \frac{\exp(-\epsilon(n^2) n^2)}{n}$.
\end{proof}

\begin{figure}[tb]
	\centering
	\includegraphics[width=0.8\textwidth]{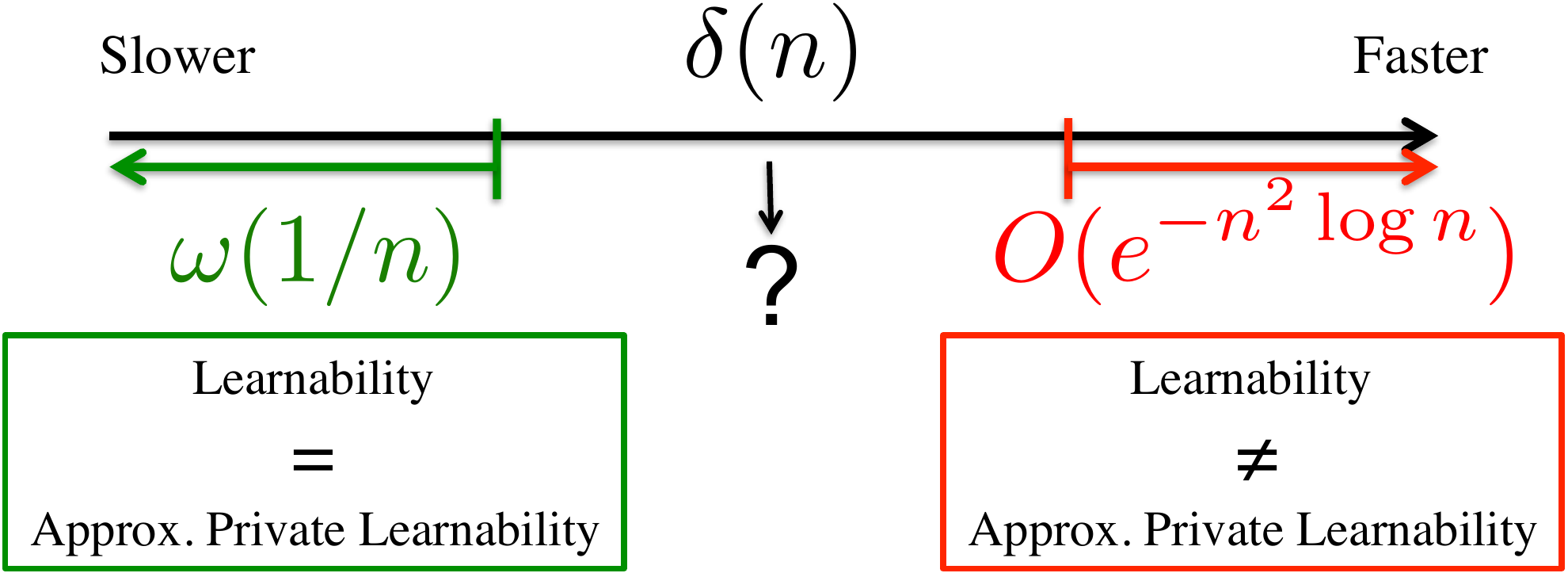}
	\caption{Illustration of Proposition~\ref{prop:phase_transition} and the open problem.}\label{fig:illus_phase_transition}
\end{figure}
The bound can be further improved to $\exp(-\epsilon(n) n)/n$ if we directly work with universal consistency on various distributions rather than through always AERM on specific data points. Even that is likely to be suboptimal as there might be more challenging problems and less favorable packings to consider.  

The point of this exposition, however, is to illustrate that $(\epsilon,\delta)$-DP alone does not close the gap between learnability and private learnability. Additional relaxation on the specified rate of decay on $\delta$ does. We now know that the phase transition occurs when $\delta(n)$ is somewhere between $\Omega(\exp(-n^2\log n))$ and $O(1/n)$; but there is still a substantial gap between the upper and lower bounds.

\section{Conclusion and future work}
In this paper, we revisited the question {\it ``What can we learned privately?''} and considered a broader class of statistical machine learning problems than those studied previously. Specifically, we characterized the learnability under privacy constraint by showing any privately learnable problems can be learned by a private algorithm that asymptotically minimizes the empirical risk for any data, and the problem is not privately learnable otherwise. This allows us to construct a conceptual procedure that privately learns any privately learnable problem.
We also propose a relaxed notion of private learnability called private $\mathfrak{D}$-learnability, which requires the existence of an algorithm that is consistent for any the distribution within a class of distributions $\mathfrak{D}$. We characterized private $\mathfrak{D}$-learnability too with a weaker notion of AERM. 
 For problems that can be formulated as penalized empirical risk minimization, we provide a sampling algorithm with a set of meaningful sufficient conditions on the geometry of the hypothesis space and demonstrate that it covers a large class of problems.
In addition, we further extended the characterization to learnability under $(\epsilon,\delta)$-differential privacy and provided a preliminary analysis which establishes the existence of a phase transition from all learnable problems being approximately private learnable to some learnable problems being not approximately private learnable at some non-trivial rate of decay on $\delta(n)$.

Future work includes understanding the conditions under which privacy and AERM are contradictory (recall that we only have one example on learning thresholding functions due to \citealt{chaudhuri2011sample}), characterizing the rate of convergence, searching for practical algorithms that generically learns all privately learnable problems, and better understanding the gap between learnability and approximate private learnability. 

\section*{Acknowledgment}
We thank the AE and the anonymous reviewers for their comments that lead to significant improvement of this paper. The research was partially supported by NSF Award BCS-0941518 to the Department of Statistics at Carnegie Mellon University, and a grant by Singapore National Research Foundation under its International Research Centre @ Singapore Funding Initiative and administered by the IDM Programme Office.


\appendix
\addappheadtotoc
\section{Proofs of technical results}
In this appendix, we provide detailed proofs to the technical results that in the main text.

\subsection{Privacy in subsampling}\label{sec:proof-subsampling}
\begin{proof}[Proof of Lemma~\ref{lem:dp_learnability}]
Let $\cA$ be the consistent $\epsilon$-DP algorithm. Consider $\cA'$ that apply
$\cA$ to a random subsample of $\lfloor \sqrt{n}\rfloor$ data points. By Lemma~\ref{lem:sampling_thm} with $\gamma = \frac{\lfloor \sqrt{n}\rfloor}{n} \leq \frac{1}{\sqrt{n}}$, we get the privacy claim. For the consistency claim, note that the given sample is an iid sample of size $\sqrt{n}$ from the original distribution.
\end{proof}

\begin{lemma}[Subsampling theorem]\label{lem:sampling_thm}
If Algorithm $\cA$ is $\epsilon$-DP for $Z\in \cZ^n$ for any $n=1,2,3,...$, then the algorithm $\cA'$ that output the result of $\cA$ to a random subsample of size $\gamma n$ data points preserves $2\gamma (e^{\epsilon}-e^{-\epsilon})$-DP.
\end{lemma}
\begin{proof}[Proof of Lemma~\ref{lem:sampling_thm} (Subsampling theorem)]
This is a corollary of Lemma~4.4 in \citet{beimel2014bounds}. To be self-contained, we reproduce the proof here in our notation.

Recall that $\cA'$ is the algorithm that first randomly subsample $\gamma n$ data points then apply $\cA$.
Let $Z$ and $Z'$ be any neighboring databases and assume they differ on the $i$th data point. Let $\cS\subset [n]$ be the indices of the random subset of the entries that are selected, and $\cR\subset [n]\backslash\{i\}$ be a index size of size $\gamma n-1$.  We apply the law of total expectation twice and argue that for any adjacent $Z$, $Z'$, any event $E\subset \cH$,
\begin{equation*}
\resizebox{\columnwidth}{!}{
	$
\begin{aligned}
&\frac{\P_{h\sim \cA'(Z)}(h\in E)}{\P_{h\sim \cA'(Z')}(h\in E)} = \frac{\gamma \P_{h\sim \cA(Z_S)}(h\in E  | i\in \cS)  + (1-\gamma)  \P_{h\sim \cA(Z_\cS)}(h\in E  | i\notin \cS)}{\gamma \P_{h\sim \cA(Z'_\cS)}(h\in E  | i\in \cS)  + (1-\gamma)  \P_{h\sim \cA(Z'_\cS)}(h\in E  | i\notin \cS)}\\
=& \frac{\sum_{\cR\in [n]\backslash \{i\}} \P(\cR)  \left[\gamma \P_{h\sim \cA(Z_\cS)}\left(h\in E  | \cS=\cR \cup\{i\}\right) + (1-\gamma)  \P_{h\sim \cA(Z_\cS)}\left(h\in E  | \cS=\cR\cup \{j\}, j\neq i\right)\right] }{\sum_{\cR\in [n]\backslash \{i\}} \P(\cR)  \left[\gamma \P_{h\sim \cA(Z'_\cS)}\left(h\in E  | \cS=\cR \cup\{i\}\right) + (1-\gamma)  \P_{h\sim \cA(Z'_\cS)}\left(h\in E  | \cS=\cR\cup \{j\}, j\neq i\right)\right]}
\end{aligned}$
}
\end{equation*}

By the given condition that $\cA$ is $\epsilon$-DP, we can replace $\cR\cup \{i\}$ with $\cR\cup\{j\}$ for an arbitrary $j$ with bounded changes in the probability and the above likelihood ratio can be upper bounded by
$$
\resizebox{\columnwidth}{!}{$
\frac{\left(\gamma e^{\epsilon} + 1-\gamma\right) \E_{\cR\in[n]\backslash\{i\}, j\neq i}\P_{h\sim \cA(Z_\cS)}\left(h\in E  | \cS=\cR\cup \{j\}\right)}{\left(\gamma e^{-\epsilon} + 1-\gamma\right) \E_{\cR\in[n]\backslash\{i\}, j\neq i}\P_{h\sim \cA(Z_\cS)}\left(h\in E  | \cS=\cR\cup \{j\}\right)} = \frac{\gamma e^{\epsilon} + 1-\gamma}{\gamma e^{-\epsilon} + 1-\gamma} = \frac{1 + \gamma(e^{\epsilon}-1)}{1+\gamma(e^{-\epsilon}-1)}.
$}
$$

By definition, the privacy loss of the algorithm $\cA'$ is therefore
$$\epsilon' \leq \log\left(1+\gamma[e^\epsilon-1]\right) - \log\left( 1+\gamma\left[e^{-\epsilon}-1\right]\right).$$

Note that $\epsilon>0$ implies that $-1\leq e^{-\epsilon}-1 < 0$ and $0<e^\epsilon-1<\infty$. The result follows by applying the property of the natural logarithm:
\begin{align*}
\log(1+x) \leq \frac{x}{2}\frac{2+x}{1+x} \leq x && \text{ for } 0\leq x<\infty\\
\log(1+x) \geq \frac{x}{2}\frac{2+x}{1+x}\geq \frac{x}{1+x} && \text{ for } -1\leq x\leq 0
\end{align*}
to upper bound the expression.
\end{proof}

\subsection{Characterization of private learnability}\label{sec:proof-characterization}
\paragraph{Privacy implies stability} Lemma~\ref{lem:EM_stability} says that an $\epsilon$-differentially private algorithm is $(e^\epsilon-1)$-stable (and also $2\epsilon$-stable if $\epsilon<1$).
\begin{proof}[Proof of Lemma~\ref{lem:EM_stability}]
Construct $Z^\prime$ by replacing an arbitrary data point in $Z$ with $z^\prime$ and let the probability density/mass defined by $\cA(Z)$ and $\cA(Z^{\prime})$ be $p(h)$ and $p^\prime(h)$ respectively,  then we can bound the stability as follows
\begin{align*}
&\left| \E_{h\sim\cA(Z)}\ell(h,z) - \E_{h\sim\cA(Z^{\prime})}\ell(h,z)\right| \\
=& \left| \int_{h}\ell(h,z)p(h)dh - \int_{h}\ell(h,z)p^\prime(h)dh\right|
= \left|\int_h\ell(h,z)(p(h)-p^{\prime}(h))dh\right| \\
\leq& \sup_{h,z}\left|\ell(h,z)\right| \int_{p(h)\geq p^{\prime}(h)} p(h)-p^{\prime}(h)dh
\leq 1\cdot \int_{p(h)\geq p^{\prime}(h)}
p^{\prime}(h)(\frac{p(h)}{p^{\prime}(h)}-1)dh\\
\leq&(e^{\epsilon}-1) \int_{p(h)\geq p^{\prime}(h)}p^{\prime}(h) dh \leq (e^{\epsilon}-1).
\end{align*}

For $\epsilon<1$ we have
$\exp(\epsilon) -1 < 2\epsilon
$.

\end{proof}

\paragraph{Stability + AERM $\Rightarrow$ consistency}
\begin{theorem}[Randomized version of {\citealt[Theorem~8]{shalev2010learnability}}]\label{thm:stable_aerm}\quad\\ If any algorithm is $\xi_1(n)$-stable and $\xi_2(n)$-AERM then it is consistent with rate $\xi(n)=\xi_1(n)+\xi_2(n)$.
\end{theorem}
\begin{proof}

We will show the following the two steps as in \citet{shalev2010learnability}
\begin{enumerate}
  \item Uniform RO stability $\Rightarrow$ On average stability $\Leftrightarrow$ On average generalization
  \item AERM + On average generalization $\Rightarrow$ consistency
\end{enumerate}
The definition of these quantities is self-explanatory.

To show that ``stability implies generalization'', we have
\begin{align*}
  &\;\Big|\E_{Z\sim \cD^n}  \left( \E_{h\sim \cA(Z)} R(h) - \E_{h\sim \cA(Z)}\hat{R} (h,Z)\right)\Big|\\
  =&
  \Big|\E_{Z\sim \cD^n}  \left( \E_{z\sim \cD}\E_{h\sim \cA(Z)}\ell(h,z) - \frac{1}{n}\E_{h\sim \cA(Z)}\sum_{i=1}^n \ell(h,z_i)\right)\Big|\\
   =& \;\Big|\E_{Z\sim \cD^n, \{z^{\prime}_1,...,z^{\prime}_n\}\sim \cD^n}  \left( \frac{1}{n}\sum_{i=1}^n \E_{h\sim \cA(Z)}\ell(h,z^{\prime}_i) - \frac{1}{n}\sum_{i=1}^n\E_{h\sim \cA(Z^{(i)})} \ell(h,z^{\prime}_i)\right)\Big|\\
   \leq&\; \sup_{Z, Z^{(i)}\in \cZ^n, d(Z, Z^{(i)})=1, z^{\prime}\in \cZ} \Big| \E_{h\sim \cA(Z)}\ell(h,z^{\prime})-\E_{h\sim \cA(Z^{(i)})}\ell(h,z^{\prime})\Big| \leq \xi_1(n)\,,
\end{align*}
where $Z^{(i)}$ is obtained by replacing the $i$th entry of $Z$ with $z'_i$.
 Next, we show that ``generalization and AERM implies consistency''. Let $ h^*\in \arg\inf_{h\in\cH} R(h)$. By definition, we have $\E_{Z\sim \cD^n}\hat{R}(h^*,Z) = R^*$. It follows that
\begin{align*}
&\;\E_{Z\sim \cD^n} [\E_{h\in \cA(Z)}R(h)  - R^*] =\E_{Z\sim \cD^n}[\E_{h\in \cA(Z)}R(h)  - \hat{R}(h^*,Z)]\\
=&\; \E_{Z\sim \cD^n}[\E_{h\in \cA(Z)}R(h) - \E_{h\in \cA(Z)}\hat{R}(h,Z)] + \E_{Z\sim \cD^n}[\E_{h\in \cA(Z)}\hat{R}(h,Z) -\hat{R}(h^*,Z)]\\
\leq&\; \E_{Z\sim \cD^n}[\E_{h\in \cA(Z)}R(h) - \E_{h\in \cA(Z)}\hat{R}(h,Z)] + \E_{Z\sim \cD^n}[\E_{h\in \cA(Z)}\hat{R}(h,Z) -\hat{R}^*(Z)]\\
\leq &\; \xi_1(n) + \xi_2(n)\,.
\end{align*}
\end{proof}

\paragraph{Privacy + AERM $\Rightarrow$ consistency}
\begin{proof}[Proof of Corrollary~\ref{corr:consistency}]
It follows by combining Lemma~\ref{lem:EM_stability} and Theorem~\ref{thm:stable_aerm}.
\end{proof}

\paragraph{Necessity}
\begin{proof}[Proof of Lemma~\ref{lem:necessity}]
We construct an algorithm $\cA'$ by subsampling the data points using a random subset of $\sqrt{n}$ and then running $\cA$. The privacy claim follows from Lemma~\ref{lem:sampling_thm} directly.

To prove the ``always AERM'' claim, we adapt the proof of Lemma 24 in \citet{shalev2010learnability}. For any fixed data set $Z\in \cZ^n$,
\begin{align*}
\hat{R}(\cA'(Z),Z) - \hat{R}^{*}(Z) &= \E_{Z'\subset Z, |Z'|=\lfloor\sqrt{n}\rfloor} \left[ \hat{R}(\cA(Z'),Z) - \hat{R}^{*}(Z) \right] \\
&= \E_{Z'\sim \text{Unif}(Z)^{\lfloor \sqrt{n}\rfloor}} \left[ \hat{R}(\cA(Z'),Z) - \hat{R}^{*}(Z) | \text{ no duplicates } \right]\\
&\leq \frac{\E_{Z'\sim \text{Unif}(Z)^{\lfloor \sqrt{n}\rfloor}} \left[ \hat{R}(\cA(Z'),Z) - \hat{R}^{*}(Z)\right]}{\P(\text{no duplicates})},
\end{align*}
where $ \text{Unif}(Z)$ is the uniform distribution defined on the $n$ points in $Z$. We need to condition on the event that there are no duplicates for the second equality to hold because $Z'$ is a subsample taken without replacements. The last inequality is by the law of total expectation and the non-negativity of the conditional expectation. But
$\mathbb P(\text{no duplicates})=\prod_{i=0}^{\lfloor \sqrt{n}\rfloor - 1}(1-i/n)\ge
1-\sum_{i=0}^{\lfloor \sqrt{n}\rfloor - 1} i/n \ge 1/2$.
By universal consistency, $\cA$ is consistent on the discrete uniform distribution defined on $Z$, so
\begin{align*}
\hat{R}(\cA'(Z),Z) - \hat{R}^{*}(Z) \leq 2\E_{Z'\sim \text{Unif}(Z)^{\lfloor n\rfloor}} \left[ \hat{R}(\cA(Z'),Z) - \hat{R}^{*}(Z)\right] \leq 2 \xi(\sqrt{n}).
\end{align*}

It is obvious that $\mathcal A'$ is consistent with rate $\sqrt{n}$ as it applies $\mathcal A$
on a random sample of size $\sqrt{n}$.
By Lemma~\ref{lem:dp_learnability}, $\mathcal A'$ is $2n^{-1/2}(e^\epsilon-e^{-\epsilon})$
differentially private.
By Corollary~\ref{corr:consistency}, the new algorithm $\cA'$ is universally consistent.
\end{proof}

\subsection{Proofs for \Cref{sec:counter_example}}\label{sec:proof-counter-example}
\begin{proof}[Proof of Proposition~\ref{pro:non-private-learnability}]
If $\cA(Z)$ is a continuous distribution, we can pick $h\in \cH$ at any point where $\cA(Z)$ has finite density and set $\cA'(Z)| z\in Z$ to be $h$ with probability $1/n$ and the same as $\cA(Z)$ with probability $1-1/n$. This breaks privacy because conditioned on two databases with $z$ or without $z$, $\cA$, the probability ratio of outputting $h$ is $\infty$.

If $\cA(Z)$ is a discrete distribution or a mixed distribution, it must have the same support of the point mass for all $Z$. Otherwise it violates DP because we need $\frac{\P_{h\in\cA(Z)}(h)}{\P_{h\in\cA(Z')}} \leq \exp(n\epsilon)$ for any $Z,Z'\in \cZ^{n}$. Specifically, let the discrete set of point mass be $\tilde{\cH}$ if $\cH\backslash\tilde{\cH}\neq \emptyset$, then we can use the same technique as in the continuous case by adding a small probability $1/n$ on $\cH\backslash\tilde{\cH}$ when $z\in Z$.

If $\tilde{\cH}=\cH$, then $\cH$ is a discrete set, if $|\cH| < n$, then by boundedness and Hoeffding, ERM is a deterministic algorithm that learns any learnable problem. On the other hand, if $|\cH| >n$, then by pigeon hole principle, there always exists a hypothesis $h$ that has probability smaller than $1/n$ in $\cA(Z)$ for any $Z\in \cZ^n$ and we can construct $\cA'$ by outputting a sample of $\cA(Z)$ if $z$ is not observed and outputting a sample $\cA(Z)| \cA(Z)\neq h$ whenever $z$ is observed.

The consistency of $\cA'$ follows easily as its risk is at most $1/n$ larger than
that of $\cA$.
\end{proof}

\subsection{Proofs for characterization of private $\mathfrak D$-learnability}\label{sec:proof-D-learnability}
\begin{proof}[Proof of Lemma~\ref{lem:necessity_mod}]
  Let $\cA'$ be the algorithm that applies $\cA$ to a random subsample of size $\lfloor\sqrt{n}\rfloor$.
If we can show that, for any $\cD\in\mathfrak{D}$,
\begin{itemize}
\item[(a)] the empirical risk of $\cA'$ converges to the the optimal population risk $R^*$ in expectation;
\item[(b)] the empirical risk of the ERM learning rule also converges to $R^*$ in expectation,
\end{itemize}
then by triangle inequality, the empirical risk of $\cA'$ must also converge to the empirical risk of ERM, i.e., $\cA'$ is $\mathfrak D$-universal AERM.

We will start with (a). For any distribution $\cD\in \mathfrak{D}$, we have
\begin{align}
&\E_{Z\sim \cD^n}\hat{R}(\cA'(Z),Z) = \E_{Z\sim \cD^n}\left[\E_{Z'\subset Z, |Z'|=\lfloor\sqrt{n}\rfloor}  \hat{R}(\cA(Z'),Z) \right] \nonumber\\
=& \E_{Z'\sim \cD^{\lfloor \sqrt{n}\rfloor}}  \left[ \frac{\lfloor \sqrt{n}\rfloor}{n}\hat{R}(\cA(Z'),Z') + \E_{Z''\sim \cD^{n-\lfloor \sqrt{n}\rfloor}}\left(\frac{n-\lfloor \sqrt{n}\rfloor}{n} \hat{R}(\cA(Z'),Z'')\right)\right]\nonumber\\
=& \E_{Z'\sim \cD^{\lfloor \sqrt{n}\rfloor}}  \left[ \frac{\lfloor \sqrt{n}\rfloor}{n}\hat{R}(\cA(Z'),Z') + \frac{n-\lfloor \sqrt{n}\rfloor}{n} R(\cA(Z')) \right]
\leq \frac{1}{\sqrt{n}} + R^*+\xi(\sqrt{n}). \label{eq:Dlearnable_proof_a}
\end{align}
The last inequality uses the boundedness of the loss function to get $\hat{R}(\cA(Z'),Z')\leq 1$ and the $\mathfrak{D}$-consistency of $\cA$ to bound the excess risk of $\E_{Z'} R(\cA(Z'))$.

To show (b), we need to exploit the assumption that the problem is (non-privately) learnable. By \citet[Theorem~7]{shalev2010learnability}, the problem being learnable implies that there exists a universally consistent algorithm $\cB$ (not restricted to $\mathfrak{D}$), that is universally AERM with rate $3\xi'(n^{\frac{1}{4}}) + \frac{8}{\sqrt{n}}$ and stable with rate $\frac{2}{\sqrt{n}}$. Moreover, by \citet[Theorem~8]{shalev2010learnability}, $\cB$'s stability and AERM implies that $\cB$ is also generalizing, with rate $6\xi'(n^{\frac{1}{4}}) + \frac{18}{\sqrt{n}}$.
Here the term ``generalizing''  means that the empirical risk is close to the population risk.
Therefore, we can establish (b) via the following chain of approximations
$$\E_{Z\sim \cD^n}  \hat{R}^{*}(Z) \explain{\approx}{\text{AERM of $\cB$}} \E_{Z\sim \cD^n} \hat{R}(\cB(Z),Z) \explainup{\approx}{\text{Generalization of $\cB$}} R(\cB(Z)) \explain{\approx}{\text{Consistency of $\cB$}} R^*.$$
More precisely,
\begin{align}
&\left|\E_{Z\sim \cD^n}\hat{R}^{*}(Z) - R^* \right| \nonumber\\
\leq& \left|\E_{Z\sim \cD^n}\hat{R}^{*}(Z) - \E_{Z\sim \cD^n} \hat{R} \right| + \left|\E_{Z\sim \cD^n} \hat{R} - R(\cB(Z),Z) \right|  + \left|R(\cB(Z),Z)- R^* \right| \nonumber\\
\leq&[3\xi'(n^{\frac{1}{4}}) + \frac{8}{\sqrt{n}}] + [6\xi'(n^{\frac{1}{4}}) + \frac{18}{\sqrt{n}}] + [3\xi'(n^{\frac{1}{4}}) + \frac{10}{\sqrt{n}}]
= 12\xi(n^{\frac{1}{4}})+\frac{36}{\sqrt{n}}.\label{eq:Dlearnable_proof_b}
\end{align}

Combine \eqref{eq:Dlearnable_proof_a} and \eqref{eq:Dlearnable_proof_b}, we obtain the AERM of $\cA'$ with rate $12\xi'(n^{1/4}) + \frac{37}{\sqrt{n}} + \xi(\sqrt{n})$ as required. The privacy of $\cA'$ follows from Lemma~\ref{lem:sampling_thm}.
\end{proof}

\subsection{Proof for Theorem~\ref{thm:general_private_learning}} \label{sec:private_beyond_uniform_convergence}
We first present the proof for Theorem~\ref{thm:general_private_learning}. Recall that the roadmap of the proof is summarized in Figure~\ref{fig:bigpic2}.

For readability, we  denote $\epsilon(n)$ by simply $\epsilon$.

Recall that the objective function is $F(h,Z)=\frac{1}{n}\sum_{i=1}^{n}\ell(h,z_i) +g_n(h)$ and the corresponding utility function $q(h,Z)=-F(h,Z)$. By the boundedness assumption, it is easy to show that if we replace one data point in any $Z$ with something else, then sensitivity
\begin{equation}\label{eq:sensitivity}
  \Delta q = \sup_{h\in \cH, d(Z,Z^{\prime})=1} \left| q(Z,h)-q(Z^{\prime},h) \right|\leq \frac{2}{n}.
\end{equation}
Then by \citet[Theorem 6]{mcsherry2007mechanism}, Algorithm~\ref{alg:EM_erm} that outputs $h\in \cH$ with $\P(h)\propto\exp(\frac{\epsilon}{2\Delta q} q(h,Z))$ naturally ensures $\epsilon$-differential  privacy.

Denote shorthand $F^*:=\inf_{f\in{\cH}} F(Z,h)$ and $q^* := -F^*$, we can state an analog of the utility theorem of the exponential mechanism in \citep{mcsherry2007mechanism}.
\begin{lemma}[Utility]\label{lemma:utility}
Assuming $\epsilon <  \log n$ (otherwise the privacy protection is meaningless anyway), if assumption A1, A2 hold for distribution $\cD$, then
\begin{equation}\label{eq:utility}
  \E_{Z\sim \cD^n}\E_{h\sim \cA(Z)} q(Z,h)\geq -\E_{Z\sim \cD^n}F^* - \frac{9[(\rho+2)\log n+\log K]}{n\epsilon}.
\end{equation}
\end{lemma}
\begin{proof}
By the boundedness of $\ell$ and $g$
$$
q(Z,h) = -\frac{1}{n}\sum_{i}\ell(h,z_i) - g_n(h) \geq -(1 + \zeta(n)).
$$
By Lemma~7 in \citet{mcsherry2007mechanism} (translated to our case),
\begin{equation}\label{eq:thm_pf_eq1}
\P_{h\sim \cA(Z)} \left[q(Z,h) < -F^* -2t\right] \leq \frac{\mu(\cH)}{\mu(\cS_t)}e^{-\frac{\epsilon}{2\Delta q} t},
\end{equation}
Apply \eqref{eq:sensitivity}, take expectation over the data distribution on both sides, and applying assumption A2, we get
\begin{equation}\label{eq:utility_proof}
\E_{Z\sim \cD^n} \P_{h\sim \cA(Z)} \left[q(Z,h) < -F^* -2t\right]  \leq K t^{-\rho} e^{-\frac{\epsilon n t}{4}} =e^{-\frac{\epsilon n t}{4} + \log K -\rho\log t} := e^{-\gamma}.
\end{equation}
Take $t=\frac{4\big[(\rho+2)\log n+\log(K)\big]}{\epsilon n}$, by the assumption that $\epsilon<\log n$, we get $\log(n t)>0$. Substitute $t$ into the expression of $\gamma$ we obtain
$$\gamma = \frac{\epsilon n}{4}t - \log K + \rho\log t
=2\log n + \rho\log(n t)\geq 2\log n,$$
and therefore
$$\E_{Z\sim \cD^n} \P_{h\sim \cA(Z)} \left[q(Z,h) < -F^* -2t\right]\leq n^{-2}.$$
Denote $\P_{h\sim \cA(Z)}[q(Z,h)<-F^* -2t] =: p$, we can then bound the expectation  from below as follows:
\begin{align}
\E_{Z\sim \cD^n}\E_{h\sim \cA(Z)} q(Z,h) \geq&  \E_{Z\sim \cD^n} (-F^* - 2t) (1-p)  + \min_{h\in \cH, Z\in \cZ^n} q(Z,h) \E_{Z\sim \cD^n} p\nonumber \\
\geq & \E_{Z\sim \cD^n}(-F^*-2t) +\left(-1-\zeta(n)\right)n^{-2}\nonumber\\
\geq & -\E_{Z\sim \cD^n}F^* -\frac{8\big[(\rho+2)\log n+\log(K)\big]}{\epsilon n} - \left(1+\zeta(n)\right)n^{-2}\nonumber\\
\geq & -\E_{Z\sim \cD^n}F^* -\frac{9\big[(\rho+2)\log n+\log(K)\big]}{\epsilon n}\nonumber.
\end{align}
\end{proof}

Now we can say something about the learning problem. In particular, the AERM follows directly from the utility result and stability follows from the definition of differential privacy.
\begin{lemma}[Universal AERM]\label{lemma:EM_AERM}
Assume \textrm{A1} and \textrm{A2}, and $\epsilon\leq \log n$ (so Lemma~\ref{lemma:utility} holds), then
$$
\E_{Z\sim \cD^{n}} \left[\E_{h\sim\cA(Z)}\hat{R}(h,Z) - \hat{R}^*(Z)\right] \leq \frac{9[(\rho+2)\log n+\log (1/K)]}{n\epsilon} + \zeta(n) .
$$
\end{lemma}
\begin{proof}
This is a simple consequence of boundedness and Lemma~\ref{lemma:utility}.
\begin{align*}
    &\E_{Z\sim \cD^{n}} \left[\E_{h\sim\cA(Z)}\hat{R}(h,Z) - \hat{R}^*(Z)\right]\\
    =&\E_{Z\sim \cD^{n}} \E_{h\sim \cA(Z)} \frac{1}{n}\sum_{i} \ell(h,z_i) -\E_{Z\sim \cD^{n}} \inf_h \frac{1}{n}\sum_i \ell(h,z_i)\\
    \leq& \E_{Z\sim \cD^{n}}\E_{h\sim \cA(Z)} \left[\frac{1}{n}\sum_{i} \ell(h,z_i) + g_n(h)\right] -\E_{h\sim \cA(Z)} g_n(h)\\
    & - \E_{Z\sim \cD^{n}}\inf_h \left[\frac{1}{n}\sum_i \ell(h,z_i) + g_n(h)\right] +\sup_h(g_n(h))\\
     =&\E_{Z\sim \cD^{n}} (-F^*-\E_{h\sim \cA(Z)} q(Z,h)) + \sup_h g_n(h)  -  \E_{h\sim \cA(Z)} g_n(h)\\
     \leq & \frac{9[(\rho+2)\log n+\log (1/K)]}{n\epsilon} + 2\zeta(n).
\end{align*}
The last step applies Lemma~\ref{lemma:utility} and $\sup_h {|g_n(h)|}\leq \zeta(n)$ as in Assumption A2 by using the fact that $\sup_h g_n(h)-\E g_n(h) \leq 2\sup_h{|g_n(h)|}$ for any distribution of $h$ the expectation is taken over.
\end{proof}

The above theorem shows that Algorithm~\ref{alg:EM_erm} is asymptotic ERM. By Theorem~\ref{lem:EM_stability}, the fact that this algorithm is $\epsilon$-differential private implies that it is $2\epsilon$-stable. Now the proof follows by applying Theorem~\ref{thm:stable_aerm} which says that stability and AERM of an algorithm certify its consistency. Noting that this holds for any distribution $\cD$ completes our proof for learnability in  Theorem~\ref{thm:general_private_learning}.

\subsection{Proofs of other technical results}

\paragraph{High confidence private learning.}

\begin{proof}[Proof of Theorem~\ref{thm:high_confidence}]
The algorithm $\cA$ privately learns the problem with rate $\xi(n)$ implies that
$$
\E_{Z\in \cD^n}\E_{h\sim\cA(Z)} R(h)-R^* \leq \xi(n).
$$
Let $h\sim \cA(Z)$ and $Z\sim \cD^n$, by Markov's inequality, with probability at least $1-1/e$,
\begin{equation*}
R(h) - R^* \leq e \xi(n).
\end{equation*}
If we split the data randomly into $a+1$ parts of size $n/(a+1)$ and run $\cA$ on the first $a$ partitions, then we get $h_j\sim \cA(Z_j)$. Then with probability at lest $1-(1/e)^{a}$, at least one of them has risk
\begin{equation}\label{eq:highconf_proof1}
\min_{j\in[a]}{R(h_j)} -R^*\leq e \xi(\frac{n}{a+1}).
\end{equation}
Since the $(a+1)$th partition are iid data, and $\ell$ is bounded, we can apply Hoeffding's inequality and union bound, so that with probability $1-\delta_1$ for all $j=1,...,a+1$
\begin{equation}\label{eq:highconf_proof2}
\hat{R}(h_j,Z_{a+1}) - R(h_j) \leq \sqrt{\frac{\log (2a/\delta_1)}{2n} }.
\end{equation}
This means that if exponential mechanism picked the one with the best validation risk it will be almost as good as the one with the best risk. Assume $h_1$ is the one that achieves the best validation risk.

Now it remains to bound the probability that exponential mechanism pick an $h\in \{h_1,...,h_a\}$ that is much worse than $h_1$.

Recall that the utility function is the negative validation risk  which depends only on the last partition $I_{a+1}$.
$$
q(X,h) = \frac{1}{n/(a+1)}\sum_{i\in I_{a+1}} \ell_i(z_i,h).
$$
 This is in fact a random function of the data because we are picking the the validation set $I_{a+1}$ randomly from the data.
Suppose we arbitrarily replace one data point $j$ from the dataset, the distribution of the output of function $q(Z,h)$ is a mixture of the two cases: $j\in I_{a+1}$ and $j\notin I_{a+1}$. Since in the first case, $q(Z,h)=q(Z',h)$ for all $h$, sensitivity for this case is $0$. In the second case, by the boundedness assumption, the sensitivity is at most $2(a+1)/n$. For the exponential mechanism guarantee $\epsilon$ differential privacy, it suffices to take the sensitivity parameter to be $2(a+1)/n$.


By the utility theorem of the exponential mechanism,
\begin{equation}\label{eq:highconf_proof3}
    \P\left[\hat{R}(h)> \hat{R}(h_1)+\frac{8(\eta\log n + \log a)}{\epsilon n/(a+1)}\right] \leq n^{-\eta}.
\end{equation}
Combine \eqref{eq:highconf_proof1}\eqref{eq:highconf_proof2} and\eqref{eq:highconf_proof3} we get
$$
\P\left[R(h) - R^*> e\xi(\frac{n}{a+1}) + \sqrt{\frac{\log(2a/\delta_1)}{2n}} + \frac{8(\eta\log n + \log a)}{\epsilon n/(a+1)}\right]\leq n^{-\eta} + \delta_1+e^{-a}.
$$
Now by appropriately choosing $\eta=\log(3/\delta)/\log n$, $a=\log(3/\delta)$, $\delta_1=\delta/3$, we get
\begin{align*}
\P\Bigg[R(h) - R^*> e\xi(\frac{n}{\log(3/\delta)+1}) + \sqrt{\frac{\log (2\log(3/\delta))+\log(3/\delta)}{2n}} \\
+ \frac{8(\log(3/\delta) + \log \log (3/\delta))}{\epsilon n/(\log(3/\delta)+1)}\Bigg]\leq \delta
\end{align*}
combine the terms and take $\epsilon=\frac{\log(3/\delta)+1}{\sqrt{n}}$, we get the bound of the excess risk in the theorem.

To get the privacy claim, note that we are applying $\cA$ on disjoint partitions of the data so the privacy parameter does not aggregate. Take the worst over all partitions, we get the overall privacy loss $\max\left\{\epsilon\left(\frac{n}{\log(3/n)+1}\right), \frac{\log(3/\delta)+1}{\sqrt{n}} \right\}$ as stated in the theorem.
\end{proof}

\paragraph{The Lipschitz example.}

\begin{proof}[Proof of Example~\ref{exmp:Lipschitz}]
Let $h^* \in \argmin_{h\in\cH} F(Z,h)$, the Lipschitz condition dictates that for any $h$,
$$
|F(h)-F(h^*)| \leq L \|h-h^*\|_p.
$$
Choose a small enough $t<t_0$ such that $h$ is in the small neighborhood of $h^*$, and we can construct a function $\tilde{F}$ that within the sublevel set $\cS_t$, such that the above inequality (when we replace $F$ with $\tilde{F}$) is equality, then for any $h\in \cS_{t_0}$, $\tilde{F}(h)\geq F(Z,h)$. Verify that the sublevel set of $\tilde{F}(h)$, denoted by $\tilde{\cS}_{t}$ always contains $\cS_t$. In addition, we can compute the measure $\mu(\tilde{\cS}_t)$ explicitly, since the function is a cone and
$$
L \|h-h^*\|_p=|\tilde{F}(h)-\tilde{F}(h^*)| =\tilde{F}(h)-\tilde{F}(h^*)\leq
t,
$$
therefore
$$
\tilde{\cS}_t = \{h\;|\; L\|h-h^{*}\|_p\leq t \}.
$$

Since $\cH$ is $\beta_p$-regular, $\mu(B\cap \cH)\ge \beta_p \mu(B)$ for any $\ell_p$ ball $B\subset \mathbb R^d$, the measure of the sublevel set can be lower bounded by $\beta_p$ times the volume of the $\ell_p$ ball with radius $t/L$ and since $\tilde{\cS}_t\subseteq \cS_t$, we have
$$
\mu(\cS_t) \geq \mu(\tilde{\cS}_t) \geq \beta_p \mu\left(B(t/L)\right)= \beta_p\left(t/L\right)^d
$$
as required.
\end{proof}

\section{Alternative proof of Corollary~\ref{corr:consistency} via \citet[Theorem~7]{dwork2014preserving}}\label{app:B}
In this Appendix, we describe how the results in \citet{dwork2014preserving} can be used to obtain the forward direction of our characterization without going through a stability argument. We first restate the result here in our notation:
\begin{lemma}[Theorem 7 in \citealt{dwork2014preserving}]\label{lem:dwork-hardt}
Let $\cB$ be an $\epsilon$-DP algorithm such that given a dataset $Z$, $\cB$ outputs a function from $\cZ$ to $[0,1]$. For any distribution $\cD$ over $\cZ$ and random variable $Z\sim \cD^{n}$, we let $\phi \sim \cB(Z)$. Then for any $\beta>0$, $\tau>0$ and $n\geq 12 \log(4/\beta)/\tau^2$, setting $\epsilon < \tau/2$ ensures
$$\P_{\phi\sim \cB(Z), Z\sim \cD^{n}}\left[\left|\E_{z\sim \cD} \phi(z) - \frac{1}{n}\sum_{z\in Z}\phi(z)\right| \geq \tau  \right] \leq \beta.$$
\end{lemma}
This lemma was originally stated to prove the claim that privately generated mechanisms for answering statistical queries always generalize.

For statistical learning problems, we can simply take the statistical query $\phi$ to be the loss function $\ell(h,\cdot)$ parameterized by $h\in \cH$. If an algorithm $\cA$ that samples from a distribution on $\cH$ upon observing data $Z$ is $\epsilon$-DP, then $\cB: Z \rightarrow \ell(\cA(Z),\cdot)$ is also $\epsilon$-DP. The result therefore reduces to that the empirical risk and population risk are close with high probability. Due to the boundedness assumption, we can translate the high probability result to the expectation form, which verifies the definition of ``generalization''.

However, ``generalization'' alone still does not imply ``consistency'', as we also need 
$$\E_{\phi \sim \cB(Z)}\frac{1}{n}\sum_{z\in Z}\phi(z) \rightarrow R^* = \min_{\phi\in\Phi} \E_{z\sim \cD}\phi(z)$$
 as $Z$ gets large, which does not hold for all DP-output $\phi$. But when $\phi = \ell(h,\cdot)$, it can be obtained if we assume $\cA$ is AERM. This is shown via the following inequality
$$
\E_{Z\in \cD^n}\E_{\phi \sim \cB(Z)} \frac{1}{n}\sum_{z\in Z}\phi(z) \rightarrow \E_{Z\in \cD^n} \min_{\phi\in \Phi} \frac{1}{n}\sum_{z\in Z}\phi(z)  \leq \E_{Z\in \cD^n} \frac{1}{n}\sum_{z\in Z}\phi^{*}(z) = \E \phi^{*}(z) = R^*,
$$
where $\phi^*=\ell(h^*,\cdot)$ and $h^*$ is an optimal hypothesis function. This wraps up the proof of consistency.

The above proof of ``consistency'' via Lemma~\ref{lem:dwork-hardt} and ``AERM'', however, leads to a looser bound comparing to our result (Corollary~\ref{corr:consistency}) when the additional assumption on $n$ and $\tau$ (equivalently $\epsilon$) is active, i.e., when $\frac{\epsilon(n)}{\log(1/\epsilon(n))} < O\left(\frac{1}{\sqrt{n}}\right)$. In this case it only implies a $\xi(n) + \frac{\log n}{\sqrt{n}}$ bound due to that $\epsilon$-DP implies $\epsilon'$-DP for any $\epsilon'>\epsilon$. Our proof of Corollary~\ref{corr:consistency} is considerably simpler and more general in that it does not require any assumption on the number of data points $n$.


This can easily lead to worse overall error bound for very simple learning problems with sufficiently fast rate. For example, in the problem of learning the mean of $X\in [0,1]$, let the loss function be $|x-h|^{10}$. Consider the $\epsilon(n)$-DP algorithm that outputs $\text{ERM} + \text{Laplace}(\frac{2}{\epsilon(n) n})$ where $\epsilon(n)$ is chosen to be $n^{-9/10}$. This algorithm is AERM with rate $\xi(n) = \frac{10!2!}{(\epsilon(n)n)^{10}} =  O(n^{-1})$.  By Corollary~\ref{corr:consistency} we get an overall rate of $O(n^{-9/10})$ while through Lemma~\ref{lem:dwork-hardt} and the argument that follows, we only get $\tilde{O}(n^{-1/2})$.


\label{sec:proofs}

\bibliographystyle{apa-good}
\bibliography{stability_privacy}
\end{document}